\documentclass{article}

\usepackage[accepted]{tmlr}

\usepackage[utf8]{inputenc}
\usepackage[T1]{fontenc}
\usepackage[hypertexnames=false]{hyperref}
\usepackage{url}
\usepackage{booktabs}
\usepackage{amsfonts}
\usepackage{amsmath}
\usepackage{amssymb}
\usepackage{xcolor}
\usepackage{graphicx}
\usepackage{pifont}

\definecolor{goodgreen}{HTML}{1B7A3E}

\usepackage{algorithm}
\let\AND\relax
\makeatletter
\@ifundefined{AND}{}{\let\AND\relax}
\makeatother
\usepackage{algorithmic}
\usepackage{multirow}
\usepackage{enumitem}

\newcommand{\dt}{\Delta_t}
\newcommand{\tauval}{\tau}
\newtheorem{proposition}{Proposition}
\newcommand{\rev}[1]{#1}

\title{Greedy Decoding Is Not Precision-Invariant: Cross-Precision Output Divergence in LLM Inference%
\thanks{An LLM-based writing assistant was used for language editing and formatting checks. All scientific claims and analyses were verified by the authors.}}

\author{%
\name Gaoyuan Du \email \href{mailto:gaoyuan.du@utk.edu}{gaoyuan.du@utk.edu} \\
\addr University of Tennessee, Knoxville \\[\baselineskip]
\AND
\name Anam Nawaz Khan \email \href{mailto:akhan59@utk.edu}{akhan59@utk.edu} \\
\addr University of Tennessee, Knoxville \\[\baselineskip]
\AND
\name Rex Zhou \email \href{mailto:rexzhe1230@gmail.com}{rexzhe1230@gmail.com} \\
\addr University of Chicago \\[\baselineskip]
\AND
\name Xiaoyang Liu \email \href{mailto:lxaoya@amazon.com}{lxaoya@amazon.com} \\
\addr Amazon \\[\baselineskip]
\AND
\name Deepayan Chakrabarti \email \href{mailto:deepay@utexas.edu}{deepay@utexas.edu} \\
\addr University of Texas at Austin \\[\baselineskip]
\AND
\name Fnu Suya \email \href{mailto:fsuya@utk.edu}{fsuya@utk.edu} \\
\addr University of Tennessee, Knoxville \\[\baselineskip]
\AND
\name Xueping Li \email \href{mailto:Xueping.Li@utk.edu}{Xueping.Li@utk.edu} \\
\addr University of Tennessee, Knoxville
}

\def\month{September}
\def\year{2026}
\def\openreview{\url{https://openreview.net/forum?id=QDOKyg7a5e}}

\begin{document}

\maketitle

\begin{abstract}
Greedy decoding from large language models is commonly treated as deterministic. We show it is not \emph{precision-invariant}: the same model, prompt, and decoding algorithm produce different outputs in BF16 versus FP16 on identical hardware. Across our evaluations of six models \rev{(1.1B--7B parameters, four families; divergence additionally characterised at 12B)} and three benchmarks, \rev{49--100\%} of prompts diverge; a single token flip often cascades into trajectory-level divergence.
We develop an \rev{empirical error-propagation analysis} and find that
\rev{22 layers of accumulated body error do not distinguish flipping from non-flipping steps; the outcome depends primarily on the top-two logit margin at the lm\_head relative to the directional perturbation between the top-two candidates.} \rev{The analysis makes five testable predictions about intervention
outcomes, including that applying \emph{more} FP32 compute (broader scope)
makes agreement \emph{worse}. The experiments match all five predictions.} The best-performing low-overhead intervention we evaluate, selective FP32 lm\_head
recomputation, triggered only when the margin falls below a
threshold, delivers $+22$--$36$pp exact agreement \rev{on A10G ($+12$--$21$pp on L4 and A100)} at $<4\%$ latency overhead \rev{in low-batch (batch size $\leq 4$) single-stream inference}. We map the applicability boundary across six models and four
batch sizes, \rev{and hypothesise that training-time precision stability is a determining factor}. The method is a partial mitigation rather than a universal determinism guarantee: its benefit vanishes when body-originated error dominates, including at batch size $\geq8$ and under end-to-end FP8 in our tests.
\end{abstract}

\section{Introduction}
\label{sec:intro}

Precision switching in LLM serving is a routine operational reality, not a hypothetical concern. \rev{Model hubs distribute a single checkpoint that serving frameworks cast to the deployment format; the same model may be served in BF16 on one replica and FP16 on another depending on library defaults, configuration flags, or operator choices.} The same model, prompt, and greedy-decoding configuration can therefore produce \emph{different} outputs depending on which \rev{numerical format is active}. Under our controlled setup, standard reproducibility tooling (\texttt{torch.use\_deterministic\_algorithms}, fixed CUBLAS workspace) makes repeated runs bit-identical within each precision but does not reconcile outputs across precisions. A system can thus be repeatable within each format while failing cross-format replay. In principle, greedy decoding is deterministic: at each step $t$, the model selects $y_t = \arg\max_v z_t(v)$.

\rev{Why does cross-format agreement matter? First, \emph{reproducibility and auditing}: when exact replay is required, an audit replica running a different numerical format may fail to reproduce the logged token sequence. Second, \emph{evaluation validity}: benchmark scores and A/B tests can conflate model quality with format-induced output drift; two evaluations of the \emph{same} checkpoint can disagree on individual predictions while reporting similar aggregate accuracy. Third, \emph{fidelity to a higher-precision reference}: we show (Appendix~\ref{app:fp32_fidelity}) that our intervention moves the BF16 arm closer to FP32 greedy inference ($+19$--$21$pp sequence agreement). \rev{FP16 alone remains closer to FP32 than BF16$+$C is, so where the serving format is a free choice, FP16 is the better single-format proxy; Intervention~C addresses the case where the format is \emph{not} a free choice.} The relevant pair of formats will change as hardware evolves, but the mechanism---a low top-two margin meeting format-dependent rounding at the lm\_head---can be tested across formats. With FP8 confined to the head projection, a format-rescaled gate achieves exact agreement on TinyLlama and $+56$pp on Qwen2.5-3B. This controlled result is not representative of end-to-end FP8: when every linear layer is quantised, divergence becomes body-dominated and the head-scope repair recovers only $+1$pp (Appendix~\ref{app:fp8}).}

Running the same model with the same prompt under BF16 vs.\ FP16 produces different outputs on $59$--$70\%$ of TinyLlama-1.1B prompts across public benchmarks (GSM8K, HumanEval, MBPP). The rate rises to $82\%$ on Qwen2.5-3B-Instruct and reaches $100\%$ on the tested BF16-saturated Qwen variants. The cross-model intervention table covers six models from $1.1$B to $7$B across four families (Llama, Qwen, Mistral, and OLMoE), and a separate 12.2B experiment characterises divergence without running the full intervention comparison.\rev{\footnote{The six models in Table~\ref{tab:cross_model_sweep} are TinyLlama-1.1B-Chat, Llama-3.2-3B-Instruct, Qwen2.5-3B-Instruct, Mistral-7B-Instruct-v0.3, OLMoE-1B-7B, and DS-R1-Distill-Qwen-7B. The 12.2B probe uses Mistral-Nemo-Instruct-2407. Two additional Qwen variants (Qwen2.5-1.5B-Instruct and Qwen2.5-7B-Instruct) exhibit BF16 saturation that precludes a meaningful intervention comparison.}}

The mechanism is more specific: hidden-state error accumulates \emph{uniformly} through the transformer body (present in every prompt regardless of outcome), and \rev{a token flips when some competitor's differential perturbation exceeds its original logit gap; empirically, such flips concentrate at steps with a small top-two margin at the lm\_head} (Section~\ref{sec:div_cond_trace}). This observation motivates the remedy: selectively recompute the lm\_head in FP32 at steps where the margin falls below a gating threshold $\tauval$.

This paper makes \rev{four} contributions:
\begin{enumerate}
    \item \rev{\textbf{A systematic characterisation of cross-precision output divergence.} We measure the phenomenon primarily at the \emph{output} level across six models ($1.1$--$7$B, four families), four benchmarks, and three GPU microarchitectures: $49$--$100\%$ of greedy generations differ between BF16 and FP16, with a mean length difference of $34$ tokens, and on Qwen2.5-3B GSM8K $19\%$ of prompts flip final-answer correctness. Numerical differences across formats are well known; our measurements quantify their output-level extent and show how one token flip can cascade through an autoregressive trajectory.}
    \item \rev{\textbf{A predictive mechanism for cross-precision greedy divergence.} A token flips between BF16 and FP16 when some competitor's differential perturbation exceeds its original logit gap. Empirically, flips concentrate where the top-two margin is small relative to the per-logit perturbation scale ($\sigma_z \approx 0.026$). Four independent experiments support this characterisation---divergence-conditioned tracing, layer grafting, KV-cache grafting, and logit-lens traceback---all converging on the lm\_head as the dominant amplification stage. The directional projection of body error onto the top-two decision direction supports lm\_head dominance (Section~\ref{sec:directional_body_error}). The mechanism replicates across three tasks and two model families.}

    \item \textbf{Five controlled interventions test the mechanism's predictions.} The mechanism predicts each intervention's outcome \emph{before} the experiment is run: integer quantization and temperature sharpening have zero effect; top-$K$ FP32 recomputation works for all $K \geq 2$; full FP32 lm\_head recomputation matches top-$K$; and extending FP32 scope beyond the lm\_head degrades agreement ($63\%\!\to\!44\%$). All five predictions are confirmed.

    \item \textbf{An applicability map across models, kernels, and batch sizes.} We evaluate the intervention on six models ($1.1$--$7$B, four families), additionally characterise divergence at $12$B, test two attention kernels, and vary batch size (bs) over $\{1, 2, 4, 8\}$ (with composition experiments extending to bs$=16$). The lift ranges from $+36$pp (Llama-3.2-3B) to $0$pp (Qwen BF16-saturated); \rev{effectiveness appears to track training-time precision stability (hypothesis)}. Two boundaries emerge: the remedy survives FlashAttention-2 in the tested setting, but the lm\_head-only scope vanishes at bs$\geq\!8$.\footnote{Throughout this paper, ``Pass@1'' denotes greedy-decoding final-answer correctness on a benchmark---e.g.\ on GSM8K, the extracted final number---and is computed under our bare-prompt format with no system prompt or chat template, then McNemar-tested in matched-pair form between Baseline and Intervention~C. \rev{For HumanEval, we additionally report execution-based pass@1 (Appendix~\ref{app:execution}): execution-outcome agreement far exceeds token-level agreement (100\% vs.\ 32\% on TinyLlama-1.1B; 97.6\% vs.\ 26.2\% on Qwen2.5-3B-Instruct, which solves half the benchmark), and Intervention~C introduces no observed pass@1 degradation on either model.} Pass@1 numbers therefore lower-bound what each model can achieve under its native chat template (Appendix~\ref{app:quality}, the chat-template results).}
\end{enumerate}

\begin{figure}[h]
\centering
\includegraphics[width=0.95\linewidth]{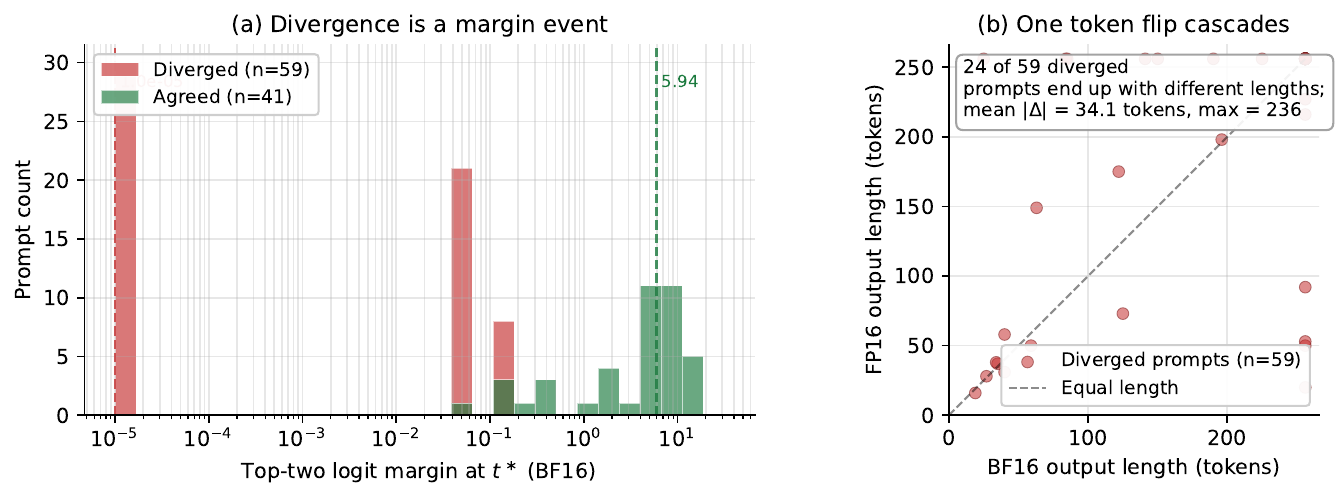}
\caption{Two empirical facts that shape this paper (TinyLlama-1.1B, GSM8K, BF16 vs.\ FP16 greedy decoding, $n\!=\!100$). \textbf{(a)} Distribution of the top-two logit margin at the measurement step $t^\ast$, for prompts whose outputs diverge (red) and for matched agreed controls (green). The two groups are bimodal across more than five orders of magnitude. \textsc{Diverged} median $\approx 10^{-5}$, with $30/59$ exactly at $0$ (perfect BF16 ties); \textsc{Agreed} median $\approx 5.9$. Divergence is a \emph{margin event} at the lm\_head, not an error-magnitude event in the transformer body. \textbf{(b)} Sequence-length scatter of BF16 vs.\ FP16 outputs on the $59$ \textsc{Diverged} prompts. A single token flip typically cascades into trajectory-level divergence: $24/59$ prompts end up with different lengths, with a mean length difference of $34$ tokens and maximum difference of $236$ tokens.}
\label{fig:hero}
\end{figure}

\section{Related work}
\label{sec:related}

As machine learning moves into audited deployments, reproducibility can be a deployment requirement in its own right, not merely a byproduct of accuracy. Floating-point non-determinism in training is classical~\citep{goldberg1991float,higham2002accuracy,zhuang2022randomness}. Framework-level determinism tools~\citep{pytorch2020determinism} are effective within a fixed format but have no effect on cross-precision divergence (Section~\ref{sec:sanity_det_flag}). BF16/FP16 serving is standard~\citep{dettmers2022gptint,micikevicius2018mixed,kalamkar2019bfloat16}, but its output-level consequences and targeted mitigation remain incompletely characterised.

The closest prior work is LayerCast~\citep{layercast2024}, which stores weights in BF16 and casts weight-relevant computation to FP32 layer by layer to reduce system-level numerical variation. \rev{LayerCast also observes a connection between small top-token gaps and numerical output changes; we build on that observation by quantifying the margin's role through grafting and gating experiments, localising the dominant correctable component to the lm\_head, and exploiting it for a selective head-scoped intervention.} We do not implement LayerCast. Our separate global FP32-compute baseline upcasts the BF16/FP16-stored model before inference and adds $+38\%$ latency on A10G, whereas the selective head-scoped intervention adds $<4\%$. \rev{PrecisionDiff~\citep{wang2026precisiondiff} also studies BF16/FP16 output disagreements, systematically identifying \emph{where} precision sensitivity appears across model components. PrecisionDiff characterises the sensitivity surface, while we analyse the decision-level mechanism, quantify each stage's contribution via grafting interventions, and derive a selective repair guided by an exact directional flip condition.} Concurrent work~\citep{he2025thinkingmachines,atil2024nondeterminism,qi2025fp16,liu2025verified,shi2025tokenprob} targets kernel- or batch-level sources orthogonal to format-level divergence, or characterises nondeterminism without localising its mechanism. \citet{messina2026temperature} formalise ``background temperature'' to characterise the effective stochasticity of floating-point perturbations at $T{=}0$; our work complements that framework with a mechanistic diagnosis and a targeted fix. To our knowledge, prior work has not combined output-level characterisation, stage-wise causal interventions, and selective lm\_head recomputation. \rev{Low-precision formats are known to produce numerical differences; we study their \emph{extent} at the output level (\rev{49--100\%} of greedy generations diverge in our evaluations), their \emph{mechanism} (late low-margin events with an lm\_head-dominated correctable component), and their \emph{partial controllability} (a margin-gated intervention recovering 22--36pp of exact agreement on A10G). Prior work on numerical (non)determinism also addresses within-format effects from batching and kernel scheduling~\citep{he2025thinkingmachines,atil2024nondeterminism}, which are distinct from the cross-format comparison studied here.}

\section{Problem formulation}
\label{sec:problem}

Consider an autoregressive LM with vocabulary $V$. At step $t$, the model computes logits $z_t = f_\theta(x, y_{<t})$ and selects $y_t = \arg\max_{v} z_t(v)$~\citep{vaswani2017attention}. In floating-point arithmetic, $z_t$ depends on the numerical precision (BF16, FP16, or FP32), tensor-parallelism degree, and GPU architecture.

We define the \emph{top-two logit gap} $\dt^{(c)} = z_t^{(c)}(v^{(1)}) - z_t^{(c)}(v^{(2)})$ under configuration $c$. When this gap is small enough that floating-point rounding across configurations can flip the argmax, different configurations select different tokens. The autoregressive conditioning then diverges, producing a cascade of length $L_{\text{cascade}} = T - t_{\text{div}}$.

We measure four quantities. The \textbf{Exact Agreement Rate} (EAR) is the fraction of prompts producing identical token sequences across configurations. The \textbf{First Divergence Step} ($t^\ast$) is the position of the first token mismatch. The \textbf{Trigger Rate} is the fraction of steps where $\dt < \tauval$ (used only when a selective intervention applies). The \textbf{Mean Gap} ($\bar{\Delta}$) averages $\dt$ over all decoding steps. For tasks with well-defined final answers, we additionally report \textbf{Semantic Agreement Rate} (SAR)---whether the extracted conclusion matches across configurations even if the full reasoning differs.

\section{Inference-time interventions for cross-precision agreement}
\label{sec:method}

\begin{figure}[h]
\centering
\includegraphics[width=0.88\linewidth, trim=0 100 0 95, clip]{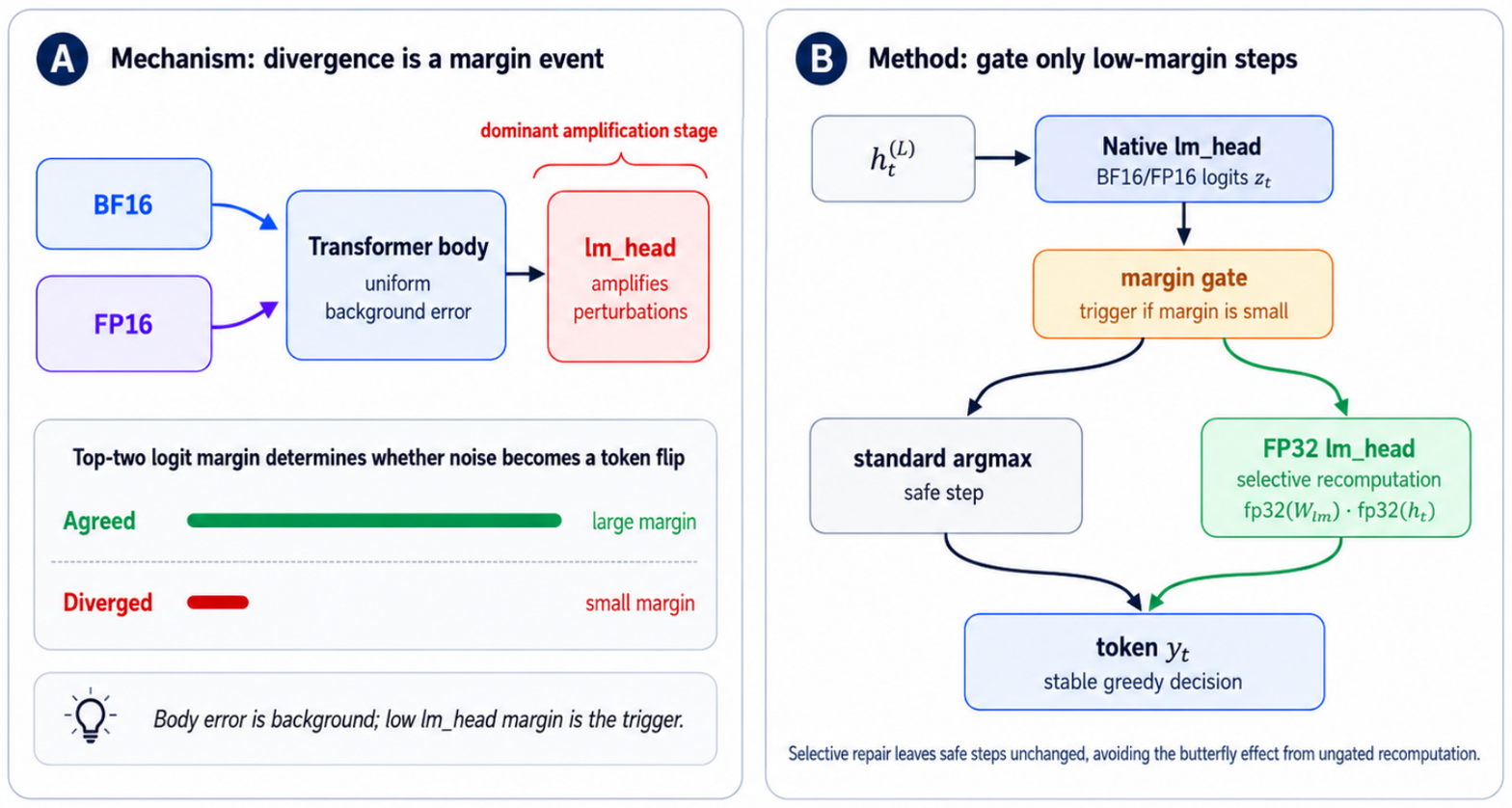}
\vspace{-6pt}
\caption{Mechanism and method. \textbf{(A)}~Cross-precision divergence is a \emph{margin event}: BF16 and FP16 accumulate uniform background error through the transformer body, and the lm\_head---the dominant amplification stage---turns that error into a token flip only when the top-two logit margin is small. Prompts with a large margin agree across precisions; prompts with a small margin diverge. \textbf{(B)}~Our method gates on this margin: the native (BF16/FP16) lm\_head logits $z_t$ are computed first, and only when the margin is small does the gate trigger a selective FP32 lm\_head recomputation ($\mathrm{fp32}(W_{\text{lm}})\!\cdot\!\mathrm{fp32}(h_t)$); large-margin (safe) steps take the standard argmax unchanged. Selective repair (${\sim}1.4\%$ of steps) leaves safe steps untouched, avoiding the butterfly effect of ungated recomputation.}
\label{fig:method_overview}
\end{figure}

The contribution of this paper is a mechanistic diagnosis of cross-precision divergence---it is a top-two margin event localised at the lm\_head---together with a controlled ablation that rules out alternative explanations. We study three inference-time interventions within this ablation: each one modifies only the final decoding step (without retraining, modifying the transformer body, or adding persistent GPU memory) but tests a different hypothesis about the mechanism. Intervention~A asks whether the issue is comparison-level (a tie-breaking artefact between BF16 and FP16 logit values that happen to be representable differently). Intervention~B asks whether a small-$K$ FP32 correction is sufficient, and is a direct empirical probe of the rank-1-vs-rank-2 prediction. Intervention~C asks whether full FP32 lm\_head recomputation eliminates divergence, and is also the most natural deployable fix. The three interventions provide a sequence of mechanism tests: success of A would support a comparison-level explanation; if A fails but B with $K=2$ succeeds, this supports a rank-1-vs-rank-2 correction mechanism in the tested gated setting; if B at $K=2$ matches C at $K=|V|$, recomputing beyond the top two provides no additional benefit in this experiment, supporting an lm\_head-localised correction with a $4000\times$ smaller intervention. Figure~\ref{fig:method_overview} illustrates the shared selective-gating envelope; the three interventions below differ only in what they manipulate inside the gated step.

\paragraph{\rev{Predictions from the analysis.}} An error-propagation \rev{analysis} (Section~\ref{sec:analysis}, $\sigma_z \approx 0.026$) predicts: A has zero effect, B and C produce comparable lifts (rank-1-vs-rank-2 only), and extending scope beyond lm\_head hurts. All predictions confirmed.

\paragraph{Selective triggering.} All interventions fire only when $\dt < \tauval = 10^{-3}$ (0.7--1.4\% of steps), keeping overhead $<\!4\%$.

\subsection{Intervention A: Integer logit quantization}

The argmax over floating-point logits can differ across precisions because BF16 and FP16 represent the same real number with different roundings; integer comparison after quantization is bit-exact. Intervention~A exploits this by quantizing logits into coarse bins of width $\tauval$ before argmax, testing whether divergence is comparison-level:

\begin{equation}
    \tilde{z}_t(v) = \text{round}\!\left(\frac{z_t(v)}{\tauval}\right), \quad y_t = \arg\max_v \tilde{z}_t(v).
\end{equation}
Any two logit values within $\tauval$ of each other map to the same integer. Ties in the integer domain are broken by selecting the smaller token index, ensuring a canonical, platform-independent decision. This is a \emph{lossy} approach that deliberately discards sub-$\tauval$ differences. The cost is negligible: one division and one round per vocabulary element.

\subsection{Intervention B: Top-$K$ FP32 recomputation with integer comparison}

Intervention~A assumes that cross-precision logits are close enough to fall within the same quantization bin. When they are not---i.e., the logit \emph{values} themselves differ across precisions---the quantization step alone cannot reconcile them. Intervention~B addresses this by first correcting the logit values through FP32 recomputation of the top-$K$ candidates, then applying integer quantization for a deterministic final comparison:
\begin{enumerate}[nosep]
    \item Extract the top-$K$ candidate indices $\mathcal{K} = \text{top-}K(z_t)$ from the native-precision logits.
    \item Recompute the lm\_head in FP32 for only those $K$ rows: $\hat{z}_t(\mathcal{K}) = \text{float32}(W_{\text{lm}}[\mathcal{K}]) \cdot \text{float32}(h_t^{(L)})$, where $W_{\text{lm}}[\mathcal{K}] \in \mathbb{R}^{K \times d}$ is the sub-matrix of lm\_head weights corresponding to the $K$ candidates.
    \item Quantize: $\tilde{z}_t(v) = \text{round}(\hat{z}_t(v) / \tauval)$ for $v \in \mathcal{K}$, then integer argmax with canonical tie-break (smallest token index wins).
\end{enumerate}
This is a \emph{hybrid} approach that combines value correction (partial FP32 recompute) with comparison stabilization (integer binning). The cost per trigger is one partial FP32 matrix multiply of size $K \times d$ ($K\!=\!8$, $d\!=\!2048$: $16\text{K}$ FLOPs), plus negligible quantization overhead---$4000\times$ cheaper than Intervention~C's full $|V| \times d$ recompute per trigger. We use $K\!=\!8$ as a conservative default; Appendix~\ref{app:topk_ablation} shows that $K\!=\!2$ suffices (all values $K \in \{2, 4, 8, 16, 32, |V|\}$ produce identical EAR), showing that, for the tested gated cases, correcting the top-two candidates is sufficient to match the full-vocabulary version.

\subsection{Intervention C: Full-vocabulary FP32 lm\_head recomputation}

The dominant source of cross-precision logit perturbation is the lm\_head: a single matrix multiplication $z_t = W_{\text{lm}} \cdot h_t^{(L)}$ projecting the $d$-dimensional hidden state onto the full vocabulary $|V|$. Because $|V|$ is large (32K--152K), this operation amplifies small hidden-state differences by a factor of $\|W_{\text{lm}}\|$. Intervention~C eliminates this amplification by recomputing the projection in FP32 over the full vocabulary:
\begin{equation}
    \hat{z}_t = \text{float32}(W_{\text{lm}}) \cdot \text{float32}(h_t^{(L)}),
\end{equation}
where both activations and weights are up-cast to FP32 on-the-fly, keeping GPU memory constant (no persistent FP32 weight copies). The recomputation scope can in principle extend to the last $L$ transformer layers, trading additional compute for deeper error correction, though our scope ablation (Section~\ref{sec:scope_ablation}) shows this \emph{degrades} agreement in practice. This is the most expensive intervention per trigger but directly targets the dominant error source identified by our layer-level analysis (Section~\ref{sec:layer_tracing}). The cost is one full FP32 matrix multiply over the entire vocabulary ($d \times |V|$).

\begin{algorithm}[t]
\caption{Inference-time intervention for cross-precision agreement}
\label{alg:masd}
\begin{algorithmic}[1]
\REQUIRE Model $f_\theta$, prompt $x$, threshold $\tauval$, max tokens $T$, intervention $\in \{\text{A, B, C}\}$
\FOR{$t = 1$ \TO $T$}
    \STATE $z_t \leftarrow f_\theta(x, y_{<t})$; \quad $v^{(1)}, v^{(2)} \leftarrow \text{top-2}(z_t)$; \quad $\dt \leftarrow z_t(v^{(1)}) - z_t(v^{(2)})$
    \IF{$\dt < \tauval$ \textbf{or} $\tauval = 0$}
        \STATE $y_t \leftarrow \text{Apply}(f_\theta, x, y_{<t}, z_t, \text{A/B/C})$
    \ELSE
        \STATE $y_t \leftarrow v^{(1)}$
    \ENDIF
    \IF{$y_t = \texttt{<eos>}$} \STATE \textbf{break} \ENDIF
\ENDFOR
\end{algorithmic}
\end{algorithm}

\subsection{Overhead and trade-off analysis}

Table~\ref{tab:strategy_tradeoff} summarizes the three interventions.

\begin{table}[t]
\caption{Inference-time interventions for cross-precision agreement. Cost is per-trigger overhead relative to a standard decoding step.}
\label{tab:strategy_tradeoff}
\centering
\small
\setlength{\tabcolsep}{5pt}
\begin{tabular}{llll}
\toprule
Intervention & Approach & Per-Trigger Cost & Mechanism \\
\midrule
A. Integer quantization & Reduce precision & $O(|V|)$ div+round & Merge small-gap logits \\
B. Top-$K$ FP32 + int compare & Increase then reduce & $O(K \cdot d)$ FP32 matmul & FP32 recompute + int merge \\
C. Full-vocab FP32 lm\_head & Increase precision & $O(d \cdot |V|)$ FP32 matmul & FP32 recompute of lm\_head \\
\bottomrule
\end{tabular}
\end{table}

With empirical trigger rates of 0.7--1.4\%, all interventions add $<$4\% latency under selective gating. The key methodological question is whether cross-precision divergence is caused by \emph{small} logit differences (within $\tauval$, addressable by integer quantization) or \emph{large} differences (ranking reversals, requiring FP32 recomputation). Our experiments answer this definitively.

We additionally evaluate two supplementary baselines in Appendix~\ref{app:additional_baselines}: \textbf{(D)~Temperature sharpening} (multiply logits by $\alpha\!>\!1$ before argmax, zero cost) and \textbf{(E)~Consensus decoding} (run BF16 and FP16 in parallel, FP32 tiebreak on disagreement, ${\sim}2\times$ cost). Neither changes the cost--agreement frontier established by A--C: sharpening has zero effect (confirming the value-level nature of divergence), and consensus achieves perfect agreement only at double the compute budget.

\section{Experimental setup}
\label{sec:setup}

\paragraph{Public benchmarks.}
Our primary evaluation uses three public benchmarks: GSM8K~\citep{cobbe2021gsm8k} (math, 8-shot CoT), HumanEval~\citep{chen2021evaluating} (code, 0-shot), and MBPP~\citep{austin2021program} (code, 3-shot).

\paragraph{Models.}
Our primary experiments (Sections~\ref{sec:divergence}--\ref{sec:strategy_results}) use TinyLlama-1.1B-Chat~\citep{zhang2024tinyllama} as the main testbed due to its fast iteration time and ability to run full ablations on a single GPU. For model sensitivity analysis (Section~\ref{sec:model_sensitivity}), we compare reasoning-tuned models (DeepSeek-R1-Distill-Qwen-1.5B/7B~\citep{deepseek2024r1}) against instruction-tuned models (Qwen2.5-1.5B/7B-Instruct~\citep{qwen2024qwen25}).

\paragraph{Configuration.}
We compare BF16 vs.\ FP16 greedy decoding on the same GPU (A10G 22GB). Max generation: 256 tokens. Fixed seeds throughout. To isolate numerical-format effects from PyTorch-level kernel non-determinism, we also run a sanity-check ablation (Section~\ref{sec:sanity_det_flag}) that varies \texttt{torch.use\_deterministic\_algorithms} across \{False, True\} with \texttt{CUBLAS\_WORKSPACE\_CONFIG=:4096:8} and \texttt{cudnn.benchmark=False}, confirming that under greedy decoding with batch size 1 and fixed seed, PyTorch kernels are already self-deterministic within a single precision and the flag has no effect on cross-precision divergence.

\paragraph{Why BF16 vs.\ FP16.} We focus on BF16 vs.\ FP16 because both are common 16-bit inference formats on modern GPUs~\citep{dettmers2022gptint,kwon2023vllm,nvidia2023trtllm}, yet they allocate exponent and mantissa bits differently ($8$/$7$ versus $5$/$10$). This pair provides a controlled test of representation-dependent output divergence. We directly test transfer to head-isolated FP8/BF16 and to FP16/FP32 on T4; applying the methodology to other format pairs remains future work.

\paragraph{Baselines.}
Greedy BF16/FP16 (standard serving precision), Greedy FP32 (oracle upper bound---eliminates all precision-induced divergence), global FP32 compute (upcast the BF16/FP16-stored model before inference), and ungated FP32 recomputation (Intervention~C with $\tauval=0$, recomputing every step without selective gating).

\section{Results}
\label{sec:results}

Unless otherwise noted, all results in this section use TinyLlama-1.1B-Chat with BF16 vs.\ FP16 greedy decoding ($n\!=\!100$ per dataset). Additional results with DeepSeek-R1 and Qwen2.5 models appear in Section~\ref{sec:model_sensitivity}.

\subsection{Sanity check: kernel-level determinism is not the cause}
\label{sec:sanity_det_flag}

Before attributing divergence to numerical format, we rule out two alternative explanations:

\paragraph{(1) PyTorch kernel non-determinism.} Toggling \texttt{torch.use\_deterministic\_algorithms}, \texttt{CUBLAS\_WORKSPACE\_CONFIG}, and \texttt{cudnn.benchmark} between ``loose'' and ``strict'' settings in independent subprocesses produces \emph{bit-identical} outputs within each precision: BF16 self-consistency is $100\%$ and the $41\%$ BF16-vs-FP16 EAR persists identically across both arms. The pattern generalises to MoE (OLMoE-1B-7B) and batch sizes $\{1, 4, 8, 16, 32\}$ (Appendix~\ref{app:sanity_ext}).

\paragraph{(2) GEMM kernel selection difference.} cuBLAS may select different GEMM algorithms for BF16 vs.\ FP16 matmuls. Two observations rule this out. First, BF16 self-consistency is $100\%$, so kernel choice is deterministic within a dtype. Second, running the same BF16 weights under native BF16 vs.\ FP32 compute (different kernels, same weights) yields $42\%$ EAR, indistinguishable from the $41\%$ baseline where both dtype and kernel differ. The divergence is caused by BF16's arithmetic rounding, not by kernel dispatch.

\subsection{Divergence is substantial across tasks}
\label{sec:divergence}

Table~\ref{tab:divergence_spectrum} summarizes BF16-vs-FP16 divergence across all datasets.

\begin{table}[t]
\caption{BF16 vs.\ FP16 divergence spectrum (TinyLlama-1.1B, $n\!=\!100$). Exact agreement (EAR) ranges from 30\% to 41\%. 95\% CI: Wilson score interval.}
\label{tab:divergence_spectrum}
\centering
\small
\begin{tabular}{llcc}
\toprule
Dataset & Domain & EAR (\%) & 95\% CI \\
\midrule
MBPP & Code & 30.0 & $[21.9, 39.6]$ \\
HumanEval & Code & 36.0 & $[27.3, 45.8]$ \\
GSM8K & Math & 41.0 & $[31.9, 50.8]$ \\
\bottomrule
\end{tabular}
\end{table}

On code generation tasks, 64--70\% of prompts produce different outputs between BF16 and FP16. On math reasoning, 59\% diverge.

\paragraph{Where does divergence happen?} The first divergence step $t^\ast$ has a median of $34$--$50$ tokens on non-Qwen models (the divergence-step analysis), with $44\%$ diverging in the first 30 tokens and a long tail past token~200. Qwen2.5-7B-Instruct is catastrophically early-biased (median $t^\ast\!=\!1$, $100\%$ first-token divergence), consistent with training-time BF16 saturation (Section~\ref{sec:model_sensitivity}). The left-skewed distribution motivates per-step gating: most divergences are localised to a small fraction of steps identifiable by the top-two margin.

\subsection{Intervention comparison: a controlled ablation}
\label{sec:strategy_results}

Table~\ref{tab:strategy_cmp} presents the central experiment: all three interventions and baselines on BF16 vs.\ FP16 ($\tauval\!=\!10^{-3}$).

\begin{table}[t]
\caption{Intervention and baseline comparison: BF16 vs.\ FP16 exact agreement (\%) (TinyLlama-1.1B, $n\!=\!100$, $\tauval\!=\!10^{-3}$). Gated FP32 lm\_head recomputation (B, C) achieves the best efficiency--agreement trade-off among the tested methods. \rev{All per-cell $95\%$ Wilson CIs have half-width ${\pm}9$--$10$pp at $n\!=\!100$; because Baseline and Intervention~C are evaluated on the \emph{same} prompts, the relevant tests are the matched-pair McNemar ($p<0.001$) and the paired bootstrap CI on the lift ($[+11,+33]$pp, Appendix~\ref{app:bootstrap}) rather than the per-cell intervals.} Full CIs for each cell are reported in Appendix~\ref{app:cis}. Appendix~\ref{app:scaled_n} replicates the Baseline / Intervention~C rows at $n\!=\!300$ (GSM8K), full-benchmark HumanEval ($n\!=\!164$), and full-split MBPP ($n\!=\!257$) to confirm the lift survives at larger sample sizes (all $n\!=\!100$ cells sit inside the scaled CIs). Appendix~\ref{app:bootstrap} additionally reports $10\,000$-sample bootstrap confidence intervals for every headline lift in the paper; the TinyLlama $+22$pp EAR gain has a $90\%$ bootstrap CI of $[+11, +33]$pp.}
\label{tab:strategy_cmp}
\centering
\small
\begin{tabular}{lcccc}
\toprule
Method & GSM8K & HumanEval & MBPP & Overhead \\
\midrule
Baseline (no intervention) & 41.0 & 36.0 & 30.0 & --- \\
\addlinespace
\multicolumn{5}{l}{\emph{Baselines}} \\
Greedy FP32 (oracle) & 100.0 & 100.0 & 100.0 & ${\sim}2\times$ mem \\
Global FP32 compute & 88.0 & 88.0 & 89.0 & ${\sim}2\times$ mem \\
Ungated FP32 recompute ($\tauval\!=\!0$) & 60.0 & 56.0 & 54.0 & ${\sim}2.5\times$ \\
\addlinespace
\multicolumn{5}{l}{\emph{Gated interventions ($\tauval\!=\!10^{-3}$)}} \\
A.\ Integer quantization & 41.0 & 36.0 & 30.0 & $<$1\% \\
B.\ Top-$K$ FP32 + integer compare & \textbf{63.0} & \textbf{61.0} & \textbf{55.0} & $<$2\% \\
C.\ Full-vocab FP32 lm\_head & \textbf{63.0} & \textbf{61.0} & \textbf{55.0} & $<$4\% \\
\bottomrule
\end{tabular}
\end{table}

Five findings emerge, building from root cause to practical recommendation:

\textbf{(1)~The FP32 oracle confirms the root cause.} Running both configurations in FP32 yields $100\%$ agreement: all divergence originates from precision differences, not from non-determinism in the decoding algorithm.

\textbf{(2)~Global FP32 compute achieves high agreement but at high cost.} Global FP32 upcasting reaches 88--89\% agreement---the highest among all methods short of the FP32 oracle---but doubles memory footprint, preventing deployment on the same hardware for larger models. Gated FP32 lm\_head recomputation (Interventions~B/C) achieves lower absolute agreement (55--63\%) but with $<$4\% overhead, making it practical for production serving.

\textbf{(3)~Ungated recomputation ($\tauval{=}0$) underperforms gating.} Recomputing at every step achieves only 54--60\% agreement at ${\sim}2.5\times$ latency, \emph{worse} than gated Interventions~B/C (55--63\% at ${<}4\%$ overhead). Global recomputation changes the logit landscape at safe steps, introducing new divergence points downstream (Section~\ref{sec:analysis}).

\textbf{(4)~Integer quantization has zero effect.} Intervention~A produces identical agreement to the baseline across all datasets, confirming that cross-precision divergence is caused by \emph{value-level} logit differences (BF16 and FP16 compute different numbers), not \emph{comparison-level} ties (same numbers compared differently by argmax).

\textbf{(5)~Gated FP32 recomputation is the best trade-off.} Interventions~B and~C achieve identical improvements (+22pp on GSM8K, +25pp on HumanEval and MBPP) with $<$4\% overhead, outperforming all baselines on the efficiency--agreement frontier. Gated FP32 lm\_head recomputation also raises SAR from 58\% to 72\% on GSM8K. The ${<}4\%$ overhead is backed by direct measurement on a single A10G (Appendix~\ref{app:overhead}). Gated C adds $+1.4\%$ latency and $+11\%$ peak memory versus the BF16 baseline; the global FP32-compute baseline adds $+38\%$ latency under the same measurement.

\paragraph{Where does the $+22$pp come from?} Table~\ref{tab:intervention_scope} stratifies the $100$ GSM8K prompts into four mutually exclusive categories based on whether gated~C triggers and whether the flip reproduces on single-step re-run. Gated~C fires on exactly $30$ prompts (those with $\dt\!<\!\tauval\!=\!10^{-3}$ at the lm\_head), and empirically fixes $22$ of them ($73\%$ success rate on triggered prompts). The $29$ prompts in the ``margin $\geq\!\tauval$'' category are where the flip is driven by accumulated upstream error---gated~C correctly declines to trigger there, because a single-step FP32 recomputation cannot reconcile two hidden states that have already drifted apart. This stratification predicts the exact headline lift ($41\%\!\to\!63\%$) without tuning and explains both where the method applies and where it does not.

\begin{table}[t]
\centering
\caption{Per-prompt stratification of when gated Intervention~C applies (TinyLlama-1.1B, GSM8K, $n\!=\!100$, $\tauval\!=\!10^{-3}$). Categories are mutually exclusive. The main-text $+22$pp EAR improvement (Table~\ref{tab:strategy_cmp}) arises because C fires on $30$ prompts (B$+$C) and successfully converts $22$ of them ($22/30\!\approx\!73\%$ of triggered prompts). Category~D ($29$ prompts) is where the flip is driven by large-error accumulation and the lm\_head margin is already decisive at $t^\ast$; gated C correctly does not trigger there.}
\label{tab:intervention_scope}
\small
\begin{tabular}{lccccccc}
\toprule
Prompt category & Count & \multicolumn{6}{c}{Distribution by $t^\ast$ (first divergence step)} \\
\cmidrule(lr){3-8}
 & & $t^\ast\!=\!0$ & $1$--$10$ & $11$--$50$ & $51$--$100$ & $>\!100$ & --- \\
\midrule
AGREED (no intervention needed) & 41 & \multicolumn{6}{c}{---} \\
Gated, single-step flip reproducible (C most effective) & 25 & 4 & 4 & 10 & 5 & 2 & --- \\
Gated, flip KV-cache-dependent (C may not fix) & 5 & --- & 2 & 1 & 1 & 1 & --- \\
Margin $\geq \tau$ (C does not trigger) & 29 & --- & 4 & 8 & 9 & 8 & --- \\
\bottomrule
\end{tabular}
\end{table}

\paragraph{Robustness checks.} We characterise divergence across six models spanning four families (Llama, Qwen, Mistral, and OLMoE-MoE; 1.1B--7B), with an additional divergence-only probe at 12B. Intervention lifts are $+36$pp on Llama-3.2-3B, $+8$--$10$pp on Mistral-7B and OLMoE, $+3$pp on Qwen2.5-3B, and $0$pp on the BF16-saturated Qwen variants. \rev{Effectiveness appears to track training-time precision stability (hypothesis; not directly validated), rather than scale alone.} The phenomenon and intervention persist with FlashAttention-2 and through 1024-token reasoning chains on MATH-500. We find no statistically significant pass@1 degradation in the matched-pair tests; a TOST analysis with margin $\delta\!=\!0.05$ establishes non-inferiority on Qwen2.5-3B at $n\!=\!600$ ($p\!=\!0.009$), while the TinyLlama result is a low-accuracy consistency check rather than informative quality evidence. C's lift degrades from $+17$pp at bs${=}1$ to $0$pp at bs${=}8$; composition with global FP32 compute recovers $+5$pp at bs${=}8$. Digit tokens are $8\times$ over-represented among flipped tokens.

\rev{
\begin{table}[t]
\centering
\caption{TOST non-inferiority analysis (GSM8K, $\delta\!=\!0.05$, $n$ per row; bare ``Solve step by step'' prompt format, hence absolute accuracies are lower than the chat-template numbers of Table~\ref{tab:qp} and are comparable only within this table). We test $H_0$: IntC accuracy $\leq$ Baseline accuracy $-\,\delta$. Rejection ($p < 0.05$) establishes that IntC is no worse than $5$pp below baseline. \rev{Interpretation note: TinyLlama's baseline accuracy ($2.3\%$) is itself below the $\delta\!=\!0.05$ margin, so a $5$pp shortfall is arithmetically impossible there and its rejection should be read as a consistency check rather than as evidence; Qwen2.5-3B ($11.7\%$) is the informative row. Paired McNemar on the same Qwen run is non-significant ($\chi^2 = 0.14$, $p = 0.708$; $b\!=\!34$ discordant pairs favouring baseline, $c\!=\!30$ favouring Intervention~C).} \rev{Qwen2.5-3B was re-run at $n{=}600$ to raise statistical power: non-inferiority is established ($p = 0.009$), with the accuracy difference's $95\%$ CI $[-4.3, +2.9]$pp lying entirely inside the $\pm 5$pp margin.}}
\label{tab:tost}
\small
\begin{tabular}{lcccccc}
\toprule
Model & $n$ & Acc (Baseline) & Acc (IntC) & Diff & TOST $p$ & Non-inferior? \\
\midrule
TinyLlama-1.1B & 300 & $2.3\%$ & $2.0\%$ & $-0.3$pp & $4.3\!\times\!10^{-5}$ & \textbf{Yes} \\
Qwen2.5-3B & \rev{600} & \rev{$11.7\%$} & \rev{$11.0\%$} & \rev{$-0.7$pp} & \rev{$\mathbf{0.009}$} & \rev{\textbf{Yes}} \\
\bottomrule
\end{tabular}
\end{table}
}

\paragraph{Counter-intuitive predictions confirmed by the mechanism.} \rev{Body-layer error does not discriminate between flipping and non-flipping prompts}: the $22$-layer TinyLlama and the $36$-layer Qwen2.5-3B have high divergence rates ($59\%$ and $82\%$) despite different body-error budgets. In these measurements, the lm\_head margin distribution is more predictive of a flip than layer count or body-error magnitude.

\subsection{Semantic agreement and gating sensitivity}
\label{sec:sar_and_gating}

Two findings extend the headline picture along orthogonal axes (full data in appendix).

\paragraph{Semantic exceeds exact agreement.} On GSM8K, $58\%$ of prompts produce the same final numerical answer despite different reasoning chains (SAR), compared to $41\%$ exact token match (EAR); Intervention~C lifts SAR to $72\%$. Token differences also need not imply functional differences on code: on Qwen2.5-3B HumanEval, execution outcomes agree on $97.6\%$ of problems while full token sequences agree on only $26.2\%$ (Appendix~\ref{app:execution}). Exact agreement therefore measures a stricter property than task-outcome agreement.

\paragraph{Downstream evaluation impact.} Cross-precision divergence can be hidden by aggregate benchmark accuracy. On Qwen2.5-3B-Instruct (GSM8K, $n\!=\!100$), $19\%$ of prompts \emph{flip correctness} between BF16 and FP16: $9$ prompts are correct only in BF16 and $10$ only in FP16, while aggregate accuracy differs by only $1$pp ($27\%$ vs.\ $28\%$). On Llama-3.2-3B-Instruct with the native chat template ($66\%$ accuracy), correctness flips drop to $2\%$. These two cases show that similar aggregate scores can mask different per-prompt outcomes; determining how this effect scales with model accuracy requires broader evaluation.

\paragraph{Selective triggering is essential.} On DeepSeek-R1-1.5B ($n\!=\!200$), gated C ($\tauval\!\in\![10^{-4}, 10^{-2}]$) achieves $+6.5$pp with $<\!4\%$ overhead, while ungated recomputation ($\tauval\!=\!0$) achieves only $+2.5$pp with $2.5\times$ latency---the ``butterfly effect'' makes blanket recomputation counterproductive (mechanism explained in Section~\ref{sec:analysis}). The threshold is insensitive across two orders of magnitude, consistent with the bimodal margin distribution.

\subsection{Model sensitivity: reasoning vs.\ instruction tuning}
\label{sec:model_sensitivity}

At matched parameter scales, reasoning-tuned models show $10$--$15$pp lower exact agreement and $1.7\times$ higher trigger rates than instruction-tuned counterparts (DeepSeek-R1 vs.\ Qwen2.5; Appendix~\ref{app:scale_sweep}). Chain-of-thought reasoning produces more uncertain intermediate steps, increasing exposure to low-margin events. Longer CoT trajectories accumulate more close-to-tipping-point decoding decisions, each precision-sensitive. This matches \rev{the scale analysis accompanying Proposition~\ref{prop:divergence_main}}: the divergence rate scales with the lower-tail mass of the margin distribution, and CoT concentrates more of this distribution at small values.

\section{\rev{Mechanism: an empirical analysis of error propagation}}
\label{sec:analysis}
\rev{We separate the empirical observation from the analysis. \emph{The observation} is that divergence concentrates at steps with a near-zero top-two margin, while the body's hidden-state error --- which is present in every prompt and grows monotonically through the layers --- does not discriminate between flipping and non-flipping steps once projected onto the decision direction. \emph{The analysis} builds on that observation: an exact algebraic condition for when a flip occurs, and an estimate of when a head-scope repair can undo one. This section provides an empirical analysis rather than formal bounds on floating-point error propagation. We give one exact algebraic condition (which is elementary), then two quantitative estimates whose assumptions we state and which we evaluate against measurements.} \rev{The model is organised around one exact result and one criterion: Proposition~\ref{prop:divergence_main} states when a token flips, and the cross-arm hidden-state gap $\Delta h$ (Table~\ref{tab:scope_regimes}) states when a head-scope repair can undo it. Between them they cover the four format pairs we measure --- BF16/FP16, FP16/FP32, head-isolated FP8, and end-to-end FP8 --- including the one in which the intervention fails. Two further failure modes, batched decode and cross-kernel variation, are governed by different mechanisms and we treat them separately.} \rev{We state one exact result (Proposition~\ref{prop:divergence_main}, with proof), develop two quantitative heuristic estimates validated against measurement, and identify the irreducible weight-truncation floor that bounds any inference-time fix.} Full derivations are in Appendix~\ref{app:theory}.

\subsection{Error propagation: from machine epsilon to logit perturbation}

Consider a transformer with $L$ layers and hidden dimension $d$, producing logits $z_t = W_{\text{lm}} \cdot h_t^{(L)}$ where $W_{\text{lm}} \in \mathbb{R}^{|V| \times d}$. Under precision $c \in \{\text{BF16}, \text{FP16}\}$, each layer introduces a rounding error bounded by $\|\delta_\ell\| \leq \epsilon_c \cdot C_\ell \cdot \|h_\ell\|$, where $\epsilon_{\text{BF16}} \approx 3.9 \times 10^{-3}$ and $\epsilon_{\text{FP16}} \approx 4.9 \times 10^{-4}$ (an $8\times$ gap due to the $7$-bit vs.\ $10$-bit mantissa). With residual connections keeping per-layer Jacobian norms $\|J_\ell\| \approx 1$, a worst-case (no-cancellation) accumulation gives \rev{$\|\hat{h}_L - h_L\| \leq L \cdot \epsilon_c \cdot \bar{C} \cdot \|h\|$}.

Empirically we measure a $\sim$5$\times$ amplification from pre-lm\_head $\ell_2$ ($0.92$) to logit $\ell_2$ ($4.63$) on TinyLlama-1.1B. The per-logit perturbation is $\|\hat{z} - z\| / \sqrt{|V|} \approx 4.63 / \sqrt{32000} \approx 0.026$. This yields the central divergence condition:

\begin{proposition}[\rev{Exact directional flip condition}]\label{prop:divergence_main}
\rev{Fix a decoding step and let $z, \hat{z} \in \mathbb{R}^{|V|}$ be the logit vectors computed by two numerical configurations. Assume each vector has a unique maximiser, let $v^{(1)} = \arg\max_v z(v)$, and write $\Delta z(v) = \hat{z}(v) - z(v)$. Then $\arg\max_v \hat{z}(v) \neq v^{(1)}$ if and only if there exists $v \neq v^{(1)}$ such that
\begin{equation}
    \Delta z(v) - \Delta z(v^{(1)}) > z(v^{(1)}) - z(v).
\end{equation}
For the original runner-up $v^{(2)}$, this inequality is exactly the condition for $v^{(2)}$ to overtake $v^{(1)}$. It is necessary for a flip whose new winner is $v^{(2)}$; together with $\hat z(v^{(2)}) > \max_{u\notin\{v^{(1)},v^{(2)}\}}\hat z(u)$, it is sufficient. Exact ties are resolved by the implementation's deterministic tie-breaking rule.}
\end{proposition}

\rev{\emph{Proof.} Under the uniqueness assumption, $\arg\max_v \hat{z}(v) \neq v^{(1)}$ iff some $v \neq v^{(1)}$ satisfies $\hat{z}(v) > \hat{z}(v^{(1)})$, i.e.\ $z(v) + \Delta z(v) > z(v^{(1)}) + \Delta z(v^{(1)})$, which rearranges to the stated condition. Taking $v=v^{(2)}$ gives the pairwise crossing condition; checking $v^{(2)}$ against all remaining coordinates identifies it as the new winner. \hfill$\square$}

\rev{\emph{On the status of this proposition.} The proposition is an algebraic rearrangement, stated separately to distinguish the proved condition from the estimates that follow. The two quantitative accounts below---the perturbation scale and the criterion for when a head-scope repair can undo a flip---are estimates with explicitly stated assumptions.}

\rev{\emph{Scale analysis (heuristic, not part of the proposition).} The proposition is exact but involves the unobservable per-coordinate perturbations $\Delta z(v)$. To obtain a practical trigger statistic we approximate their typical magnitude by the vocabulary-wide RMS $\sigma_z = \|\hat{z}-z\|/\sqrt{|V|} \approx 0.026$ (TinyLlama-1.1B), giving the directional difference a scale of $\approx \sqrt{2}\,\sigma_z \approx 0.037$. This motivates the margin-gated trigger of Intervention~C: flips concentrate where $\Delta_t$ is small relative to this scale. The approximation is deliberately heuristic---per-coordinate perturbations are neither Gaussian nor independent---and we assess its adequacy purely empirically: the directional criterion classifies flip-vs-no-flip correctly on $88$--$89\%$ of steps across two model families (Section~\ref{sec:directional_body_error}), and \textsc{Diverged} prompts show median margin $0.0$ vs.\ $2.69$ for \textsc{Agreed} ($260\times$ mean gap). An alternative sufficient-condition formulation via an $\ell_\infty$ bound is also available---no flip can occur if $2\|\hat{z}-z\|_\infty < \Delta_t$, since $|\Delta z(v^{(2)}) - \Delta z(v^{(1)})| \leq 2\|\hat{z}-z\|_\infty$---and Intervention~C's gate can be read as enforcing this safe condition at a per-format calibrated scale; we use the RMS form in the main text because $\|\hat{z}-z\|_\infty$ is dominated by outlier coordinates that rarely coincide with the top-two candidates, so the $\ell_\infty$ bound is safe but looser than the RMS-scale heuristic as a trigger statistic.}

\emph{Empirical validation.} \rev{Body-layer L2 is indistinguishable between \textsc{Diverged} and \textsc{Agreed} groups while the top-two margin separates them by orders of magnitude (Table~\ref{tab:layer_div_cond}), confirming that the margin---not the error magnitude---determines divergence.}

\subsection{Why FP32 recomputation works}
\label{sec:why_fp32}

During autoregressive generation with a KV cache, each decoding step processes only the \emph{new} token through the transformer body; the prefix hidden states are retrieved from cache without recomputation. The fresh rounding error introduced at step $t$ is therefore \rev{of order $\epsilon_c \cdot C_{\text{step}}$}---dominated by the single lm\_head matmul, not the full $L$-layer body.

\rev{\paragraph{Why FP32 recomputation reconciles the paths.} The starting point is an immediate observation, not a theorem: given the same hidden state $h_t^{(L)}$, the deterministic FP32 matmul $\text{float32}(W_{\text{lm}}) \cdot \text{float32}(h_t^{(L)})$ produces bit-identical outputs regardless of whether the upstream path was BF16 or FP16. We keep this observation explicit precisely \emph{because} it is trivial: it is what makes the intervention's reach predictable. Since the FP32 recomputation is deterministic, the two arms agree at a gated step \emph{if and only if} the hidden states they feed it agree. At a gated step, therefore, the intervention's reach is governed by the cross-arm hidden-state gap $\Delta h$ rather than by any property of the head. \rev{We use this as a screening quantity, not as a law: it orders the four format pairs we test into regimes (Table~\ref{tab:scope_regimes}), but it does not cover batching or kernel variation, and we have tested no configuration in the wide gap between the working and failing regimes.} The substantive, empirically testable claim is quantitative: at gated steps the cross-path logit difference should drop from $\sigma_z \approx 0.026$ to the \emph{body-originated residual} only, which the error model of Section~\ref{sec:scope_sufficiency} heuristically estimates as
\begin{equation}
    \sigma_z^{(\text{after C})} \approx \underbrace{\frac{L \cdot \alpha}{\sqrt{d}} \cdot \sigma_z}_{\text{body residual ($\sim$2.4\% of $\sigma_z$)}} \approx 6 \times 10^{-4}.
\end{equation}
For prompts whose divergence is entirely lm\_head-originated (body hidden states are effectively identical between precisions), this residual vanishes and C achieves exact reconciliation; the achievable EAR lift then equals the fraction of such prompts. This quantitative account is validated by the reconciliation data: the $+22$pp improvement matches the fraction of prompts whose divergence originates from single-step margin events---$30/100$ prompts enter the gate (category~B of Table~\ref{tab:intervention_scope}), of which $22$ are successfully reconciled ($73\%$ success rate within scope).}

\subsection{\rev{When lm\_head-only scope is sufficient, and when it is not}}
\label{sec:scope_sufficiency}

The analysis above shows C \emph{works}; we now show why the lm\_head is the dominant error source---making body-layer FP32 unnecessary.

\rev{\paragraph{Heuristic estimate: why the lm\_head dominates.} Under three explicit assumptions---(i) per-layer Jacobian norms $\approx 1$ maintained by residual connections, (ii) rounding errors accumulate without systematic cancellation across layers (conservative, worst-case linear in $L$), and (iii) the $d$ per-element rounding errors of the lm\_head dot product accumulate as $\sqrt{d}$ (standard random-sign model)---the body-to-lm\_head perturbation ratio is approximately $L\alpha/\sqrt{d}$, where $\alpha = \|F_\ell(h)\| / \|h\| \ll 1$ is the residual ratio:
\begin{equation}
    \frac{\sigma_z^{(\text{body})}}{\sigma_z^{(\text{lm\_head})}} \approx \frac{L \cdot \alpha}{\sqrt{d}}.
\end{equation}}

\emph{Argument.} In KV-cached decoding at step $t > 0$, only the new token passes through $L$ body layers. The residual connection $h^{(\ell)} = h^{(\ell-1)} + F_\ell(h^{(\ell-1)})$ confines rounding error to the $F_\ell$ branch, whose output norm is $\alpha \cdot \|h\|$ with $\alpha \ll 1$ (a necessary condition for training stability). We bound the body contribution conservatively (worst-case linear accumulation): after $L$ layers the body error is \rev{at most $L \cdot \alpha \cdot \epsilon_c \cdot \|h\|$}. By contrast, the lm\_head computes a $d$-dimensional dot product on $h_t^{(L)}$ at its full norm; the $d$ rounding errors accumulate to a typical magnitude \rev{$\approx \sqrt{d} \cdot \epsilon_c \cdot \|h\|$ (a random-sign estimate, \emph{not} a bound; the worst case is $d \cdot \epsilon_c \cdot \|h\|$, attained only if every rounding error shares a sign)}. \rev{This ratio deliberately combines a worst-case numerator with a typical-case denominator. Under the random-sign assumption, this choice biases the diagnostic toward assigning greater relative importance to body error. Because the denominator remains an estimate rather than a bound, however, the ratio is not a formal bound on either contribution. Adversarial weight configurations in which the $d$ products share a sign would make the head term $\sqrt{d}$ times larger and the ratio correspondingly smaller. We therefore report the ratio only as an order-of-magnitude estimate and assess it empirically: it matches the measured decomposition (${\sim}97.6\%$ head-originated in the storage-vs-computation ablation) and the ordering of intervention lifts across the four precision configurations in Table~\ref{tab:scope_regimes}.} The actual body accumulation may be sub-linear (statistical cancellation across layers would give $\sqrt{L}$), which would reduce the estimated body-to-head ratio further under the same assumptions. The resulting estimate is $L \cdot \alpha / \sqrt{d}$.

\emph{Numerical evaluation.} For TinyLlama-1.1B ($L\!=\!22$, $d\!=\!2048$, $\alpha \approx 0.05$ measured as the mean per-layer ratio $\|F_\ell(h)\|/\|h\|$ across $22$ layers and $100$ GSM8K prompts via the unconditional layer trace in Appendix~\ref{app:unconditional_trace}):
\begin{equation}
    \frac{22 \times 0.05}{\sqrt{2048}} = \frac{1.1}{45.3} \approx 0.024.
\end{equation}
Under this heuristic, the body contributes ${\sim}2.4\%$ of the per-logit perturbation and the lm\_head contributes ${\sim}97.6\%$. FP32 recomputation of the lm\_head therefore targets the estimated dominant source, consistent with the $+22$pp lift. Extending the FP32 scope to body layers addresses only the estimated residual while introducing the butterfly effect (Section~\ref{sec:scope_ablation}), which is why scope extension empirically \emph{hurts} (Appendix~\ref{app:scope_main}: RMSNorm+lm\_head drops EAR from $63\%$ to $44\%$ on TinyLlama; Appendix~\ref{app:scope} replicates the result on DeepSeek-R1-7B with $\tauval\!=\!0$). Full derivation and step-0 (prefill) edge case in Appendix~\ref{app:theory}.

\rev{\paragraph{A conditional form that covers every precision pair we test.} The ratio above is derived for two paths that run the \emph{same} body in \emph{different} low precisions, which is the BF16-vs-FP16 setting. Our FP8 experiments (Appendix~\ref{app:fp8}) sit outside that assumption, so we state the estimate in the conditional form that governs all three configurations. The relevant quantity is not the head's share of $\sigma_z$ but the \emph{cross-arm hidden-state gap} $\Delta h = \|h^{(A)} - h^{(B)}\| / \|h\|$, because Intervention~C works by making the two arms compute \emph{identical} logits: at a gated step both evaluate $\text{float32}(W_{\text{lm}}) \cdot \text{float32}(h)$, which coincide only to the extent that $h$ coincides. Writing $n_q$ for the number of quantized linear operations per layer and $\epsilon_A, \epsilon_B$ for the two arms' body precisions,
\begin{equation}
\Delta h \;\approx\; \max(n_q, 1)\, L\, \alpha\, |\epsilon_A - \epsilon_B|,
\label{eq:dh}
\end{equation}
and C can reconcile a gated step only when $\Delta h$ is small relative to the margin scale $\sigma_z$. This yields an ordering that matches every configuration we measure:

\begin{table}[h]
\centering
\caption{\rev{The conditional form of the scope estimate, evaluated for all four \emph{format} configurations we test (the batched-decode and cross-kernel boundaries are governed by a different mechanism; see text) (TinyLlama-1.1B: $L{=}22$, $d{=}2048$, $\alpha{\approx}0.05$, $n_q{=}7$ linear ops per layer). $\Delta h$ is the cross-arm hidden-state gap from Eq.~\ref{eq:dh}. \rev{The ratio column uses a single reference scale, the BF16-vs-FP16 value $\sigma_z \approx 0.026$, so that the four rows are comparable on one axis; the format's own $\sigma_z$ is larger for FP8 (Appendix~\ref{app:fp8}) and smaller for FP16-vs-FP32, and using per-format values would compress the spread without changing which regime each row falls in.}}}
\label{tab:scope_regimes}
\small
\begin{tabular}{llccl}
\toprule
Configuration & Body across arms & $\Delta h$ & $\Delta h / \sigma_z$ & Predicted / observed \\
\midrule
FP8 at head only & bit-identical & $0$ & $0$ & complete / $+70$pp, EAR $100\%$ \\
FP16 vs.\ FP32 (T4) & same body, two precisions & $0.0005$ & ${\sim}0.02$ & head-dominated / $+7$pp \\
BF16 vs.\ FP16 (main) & same body, two precisions & $0.004$ & ${\sim}0.14$ & head-dominated / $+22$pp \\
End-to-end FP8 vs.\ BF16 & body quantized in one arm & $0.45$ & ${\sim}17$ & body-dominated / $+1$pp \\
\bottomrule
\end{tabular}
\end{table}

The four configurations span nearly three orders of magnitude in $\Delta h / \sigma_z$, and the \emph{regime} tracks that ordering: complete reconciliation when the body is shared ($\Delta h = 0$), a positive lift whenever the body gap stays well below the margin scale ($\Delta h/\sigma_z \leq 0.14$), and a negligible one once it exceeds that scale ($17$). \rev{The raw lift is not monotone in $\Delta h$ and we do not claim it is: FP16-vs-FP32 yields $+7$pp at $\Delta h/\sigma_z = 0.02$ against a baseline already at $88\%$, i.e.\ it closes $58\%$ of the available headroom, whereas BF16-vs-FP16 yields $+22$pp from a $41\%$ baseline, closing $37\%$. What $\Delta h$ orders is which regime a configuration falls in, not the size of the lift within a regime.} The criterion also explains why the two pairs with the \emph{largest} raw precision gaps behave oppositely: FP16-vs-FP32 has a large $\epsilon$ ratio but a tiny $\Delta h$ because FP32 carries the untruncated weights, whereas end-to-end FP8 has both a large $\epsilon$ and $n_q{=}7$ quantized operations per layer. The $L\alpha/\sqrt{d} \approx 0.024$ figure of the previous paragraph is the special case $n_q{=}1$, $\epsilon_A/\epsilon_B$ = BF16/FP16; it should not be read as a constant. This gives an operational test for a new \emph{format pair}: measure $\Delta h$ between the two serving configurations on a handful of prompts and compare it to the margin scale. Measured on GSM8K ($n{=}50$, final-token hidden state), $\Delta h = 0.0105$ for TinyLlama BF16-vs-FP16 and $0.0219$ for Qwen2.5-3B under the same pair --- about twice the body gap. \rev{We note two limits on reading that comparison causally. The predicted value from Eq.~\ref{eq:dh} for the TinyLlama pair is $0.004$, so the estimate is a factor $2.6$ below measurement. And Qwen's lift under this same pair and the same $\Delta h$ is $+3$pp on A10G but $+14$pp on A100 and $+20$pp on L4 (Appendix~\ref{app:l4_replication}), so $\Delta h$ cannot be what sets the lift's magnitude for that model; the kernel path matters at least as much.} \rev{We are explicit about the criterion's scope. It applies to \emph{format} pairs evaluated at a fixed batch size, and it does not explain the batched-decode boundary: measuring bs$=$1 against bs$=$8 in the same precision gives $\Delta h = 0.0127$, statistically indistinguishable from the working BF16-vs-FP16 configuration, yet the lift there degrades to zero. Batching fails for a different reason, which we treat separately (Section~\ref{sec:ungated}): a changed reduction order produces a \emph{third} trajectory rather than a larger gap between two, so gated recomputation introduces new divergence instead of removing it. The criterion likewise takes the kernel implementation as fixed --- the $+3$/$+14$/$+20$pp spread for Qwen across A10G, A100 and L4 (Appendix~\ref{app:l4_replication}) shows the effective $\epsilon$ is itself implementation-dependent, a dimension the estimate does not cover.}}

\subsection{Why ungated recomputation hurts}
\label{sec:ungated}

Ungated recomputation ($\tauval=0$) produces a \emph{third} trajectory at safe steps, introducing new divergence downstream (``butterfly effect''). Result: $+2.5$pp vs.\ $+6.5$pp for gated on DeepSeek-R1-1.5B. Gating is mechanistically necessary.

\subsection{Empirical validation: layer-level divergence tracing}
\label{sec:layer_tracing}
\label{sec:div_cond_trace}

We trace BF16-vs-FP16 hidden-state L2 through all 22 layers on $n\!=\!100$ GSM8K prompts, comparing \textsc{Diverged} ($n{=}59$) and \textsc{Agreed} ($n{=}41$) at their respective measurement steps. Key finding (Table~\ref{tab:layer_div_cond}): \textbf{body-layer L2 is indistinguishable between groups}, but the \textbf{top-two margin differs by $150\times$} (0.039 vs.\ 5.733). Four independent causal experiments (layer grafting, KV-cache grafting, logit-lens traceback, Qwen2.5-3B replication; Appendices~\ref{app:graft}--\ref{app:top2_trace}) confirm: the flip is produced at the lm\_head.

\begin{table}[t]
\caption{Divergence-conditioned trace (TinyLlama-1.1B, $n\!=\!100$).}
\label{tab:layer_div_cond}
\centering\small
\begin{tabular}{lccc}
\toprule
& \textsc{Diverged} & \textsc{Agreed} & Gap \\
\midrule
Body L2 (layer 21) & 1.022 & 1.135 & $-$0.113 \\
Logit L2 & 4.928 & 5.559 & $-$0.631 \\
\textbf{Top-two margin} & \textbf{0.039} & \textbf{5.733} & \textbf{$-$5.694} \\
Flip rate & 58\% & 0\% & --- \\
\bottomrule
\end{tabular}
\end{table}

\subsection{Summary}
\label{sec:scope_ablation}
\label{sec:store_vs_compute}

The analyses above converge on a three-factor causal picture. Cross-precision greedy divergence requires the simultaneous presence of three conditions---remove any one and the outputs agree:

\begin{enumerate}[nosep]
    \item \textbf{Low arithmetic precision at the lm\_head.} BF16's 7-bit mantissa introduces per-element rounding of magnitude $\epsilon_{\text{BF16}} \approx 3.9 \times 10^{-3}$, $8\times$ larger than FP16's $4.9 \times 10^{-4}$, so the two precisions compute materially different logits at every step (per-logit perturbation $\sigma_z \approx 0.026$; see Section~\ref{sec:layer_tracing}).
    \item \textbf{The lm\_head amplifies hidden-state error.} The $d \times |V|$ projection amplifies body-layer $\ell_2$ error by $\sim$5$\times$ ($0.92 \to 4.63$); without this amplification, per-logit perturbations would sit at $\sim$0.005, below typical margins.
    \item \textbf{Greedy decoding has zero tolerance for rank flips.} Unlike sampling, greedy argmax converts any logit-ranking reversal into a hard token difference, and a single flipped token cascades through autoregressive conditioning (mean length difference $34$ tokens, Figure~\ref{fig:hero}).
\end{enumerate}

Intervention~C breaks condition~(1) at the dominant correctable stage. Recomputing the lm\_head in FP32 ($\epsilon_{\text{FP32}} \approx 6 \times 10^{-8}$) reduces the head-originated perturbation from $\sim$0.03 to $\sim$10$^{-5}$. Conditions~(2) and~(3) remain unchanged: the lm\_head still amplifies its input and greedy decoding still selects the argmax, but the corrected head contributes much less arithmetic error. The storage-vs-computation ablation supports this decomposition quantitatively: upgrading only the matmul compute path to FP32 (while keeping storage at BF16) drives EAR from $0.41$ to $0.82$, and the corresponding FP32-compute / FP32-compute pair sits at $0.88$ rather than $1.00$, indicating that ${\sim}87\%$ of the reachable EAR gap is BF16-arithmetic-origin (addressable by C) while a residual $\sim\!12$pp is associated with weight truncation. In the evaluated ablations, extending the FP32 scope beyond lm\_head (to RMSNorm or the last transformer layer) reduces EAR from $63\%$ to $44$--$52\%$ (Appendix~\ref{app:scope_main}); this shows that broader recomputation is not automatically better and can introduce new downstream trajectory differences.

\section{Discussion}
\label{sec:discussion}

\subsection{Practical implications}
\label{sec:discussion_practical}

Intervention~C is intended for low-batch greedy decoding (bs${\leq}4$), a regime that can arise in code completion, structured-output generation, benchmark evaluation, and audit-oriented replay. In our A10G measurements it is a low-overhead intervention ($<\!4\%$; Section~\ref{sec:strategy_results}), but its benefit should be measured on the target model and platform. Two scope boundaries apply: (i)~models with FP16-body NaNs or body-dominated divergence require a body-scope or training-time remedy, and (ii)~at bs$\geq\!8$ the head-only intervention provides no lift by itself in our tests and must be composed with a body-scope method (Appendix~\ref{app:layercast_composition}). \rev{The effect remains positive with both attention kernels tested (eager and FlashAttention-2) and on A10G, L4, and A100, although the lift varies substantially by GPU (Appendix~\ref{app:l4_replication}). Precision and kernel configuration should therefore be recorded explicitly rather than assuming platform-independent agreement.}

\rev{\subsection{What does partial agreement provide?}
Exact whole-sequence agreement is a strict metric: one differing token anywhere in a 256-token generation counts as disagreement. In the tested BF16/FP16 setting, Intervention~C raises this metric from 30--41\% to 55--63\% at $<4\%$ A10G latency overhead, but it does not provide bit-exact reproduction. The practical value is application-dependent. On Qwen2.5-3B HumanEval, execution outcomes agree on $97.6\%$ of problems while token sequences agree on $26.2\%$ (4 of 164 execution outcomes flip; Appendix~\ref{app:execution}); on GSM8K, however, Qwen2.5-3B shows a 19\% cross-precision correctness-flip rate. These results rule out treating either token disagreement or task neutrality as universal proxies for the other. The TOST analysis establishes non-inferiority within a 5pp margin on Qwen2.5-3B ($n\!=\!600$, $p\!=\!0.009$), while TinyLlama's very low baseline accuracy makes its non-inferiority result only a consistency check (Table~\ref{tab:tost}).

The head-isolated FP8 experiment is a controlled upper-bound diagnostic: sharing the BF16 body and weight storage removes body-originated divergence and the BF16/FP16 weight-truncation floor by construction, after which a format-rescaled gate reaches 100\% exact agreement on TinyLlama and gains $+56$pp on Qwen2.5-3B. It is not representative of end-to-end FP8 serving. When every linear layer is quantised to FP8, the same head-scope repair yields only $+1$pp (Appendix~\ref{app:fp8}). Applications requiring bit-exact replay should pin the full numerical configuration or use an end-to-end higher-precision reference rather than relying on Intervention~C alone.}

\subsection{Practitioner decision procedure}
\label{sec:discussion_decision}

The deployment decision tree (Section~\ref{sec:scope_ablation}) summarises the action tiers. The screening metric, the coefficient of variation of trigger density ($\text{CV}(\text{TD})$), requires only a single BF16 run (${\sim}5$--$10$ minutes, no FP16 comparison needed). Retrospectively, it separates the models in our study into the observed action tiers; across the five models with a defined lift comparison, $\text{CV}(\text{TD})$ is associated with Intervention~C's EAR lift (Spearman $\rho=0.90$, Pearson $r=0.84$). This is an in-sample screening heuristic, not a validated cross-model predictor, and should be recalibrated on the target hardware.

\subsection{Limitations}
\label{sec:discussion_limitations}

We characterise divergence across six models from four families at scales from 1.1B to 7B and include a divergence-only probe at 12B; the full intervention comparison is not evaluated at either the 12B or 70B+ scale. Most experiments use an A10G, with selected headline replications on L4 and A100 and an FP16/FP32 probe on T4. The direction of the intervention effect is consistent in these replications, but its magnitude is hardware-dependent. C is evaluated for greedy decoding and is effective by itself only at bs$\leq\!4$ in our batch-size sweep; at bs$\geq\!8$ and under end-to-end FP8, divergence becomes body-dominated and a head-only repair is insufficient. We do not evaluate sampling-based decoding. C does not guarantee full agreement, and the ${\sim}12$pp BF16/FP16 weight-storage gap measured in our ablation cannot be removed by changing only the lm\_head arithmetic.

\section{Conclusion}

LLM greedy decoding is not precision-invariant: \rev{49--100\%} of prompts diverge between BF16 and FP16 across the models and benchmarks we test \rev{(59--82\% on the four models that are neither BF16-saturated nor at the low end)}. Divergence concentrates at low top-two margins relative to the directional perturbation, with a $150\times$ gap between the measurement-step margins of divergent and convergent prompts in the headline trace. Gated FP32 lm\_head recomputation raises exact agreement by $+22$--$36$pp on A10G at ${<}4\%$ latency overhead and is the best-performing low-overhead intervention among those evaluated. Selected L4 and A100 replications yield positive but hardware-dependent lifts of $+12$--$21$pp. The intervention is evaluated up to the 7B scale (with divergence additionally characterised at 12B), on dense and MoE models, and is limited to low-batch greedy serving. In the controlled head-isolated FP8 setting, a format-rescaled gate reaches exact agreement on TinyLlama and gains $+56$pp on Qwen2.5-3B; in end-to-end FP8, where body error dominates, the gain is only $+1$pp. Overall, the results identify low-margin lm\_head events as an important and partially controllable source of cross-precision divergence, while also delineating regimes in which a head-scope intervention is insufficient.

\subsubsection*{Broader impact}
Cross-precision reproducibility can matter when exact replay, audit trails, or per-example evaluation records are required. Our intervention narrows the measured agreement gap at ${<}4\%$ A10G latency overhead in low-batch (bs${\leq}4$) greedy inference and requires no model retraining, but it does not guarantee identical outputs. We see no direct misuse pathway introduced by the method: it changes numerical consistency rather than model capability. Reproducibility is not correctness; a repeatable but wrong model remains wrong, and this work is not a substitute for accuracy, safety, or domain-specific validation.

\bibliographystyle{plainnat}

\newpage
\appendix
\renewcommand{\thesection}{\Roman{section}}

\section*{Appendix overview}
\label{app:overview}

The supplementary material is organised into six groups, each self-contained:
\begin{itemize}[nosep, topsep=0.3em, leftmargin=1.2em]
  \item \textbf{\rev{Analysis details}} (App.~\ref{app:theory}): \rev{the proof context for the exact flip condition and full derivations of the heuristic error-propagation estimates}.
  \item \textbf{Mechanism evidence} (App.~\ref{app:unconditional_trace}--\ref{app:top2_trace}): four causal experiments isolating divergence to the lm\_head.
  \item \textbf{Cross-model robustness} (App.~\ref{app:qwen3b_trace}--\ref{app:long_reasoning}): replication across six models, four families, math and code benchmarks.
  \item \textbf{Production-stack robustness} (App.~\ref{app:sanity_ext}): FlashAttention, batch sizes, global-FP32 composition, overhead.
  \item \textbf{Ablations} (App.~\ref{app:scope}): scope, threshold, top-$K$, alternative interventions.
  \item \textbf{Statistical validation and reproducibility} (App.~\ref{app:cis}--\ref{app:reproducibility}): CIs, scaled replication, code.
\end{itemize}

Cross-references between appendix sections are kept to a minimum.

\section{\rev{Error-propagation analysis: derivations and assumptions}}
\label{app:theory}

\rev{This appendix presents a heuristic error-propagation model. With the exception of Proposition~\ref{prop:divergence_main} (exact), the estimates here rest on the stated simplifying assumptions and are validated empirically rather than proven; counterexamples to the asymptotic rates can be constructed under adversarial weight configurations.} \rev{The model predicts the empirical findings in the main text and characterises the scope boundary of Intervention~C.}

\subsection{Setup and notation}

Let $f_\theta$ denote an $L$-layer autoregressive transformer with hidden dimension $d$ and vocabulary size $|V|$. At decoding step $t$, the model computes:
\begin{align}
    h_t^{(\ell)} &= \text{Layer}_\ell(h_t^{(\ell-1)}), \quad \ell = 1, \ldots, L, \\
    z_t &= W_{\text{lm}} \cdot h_t^{(L)}, \quad W_{\text{lm}} \in \mathbb{R}^{|V| \times d}, \\
    y_t &= \arg\max_v z_t(v).
\end{align}
Under floating-point precision $c$ with machine epsilon $\epsilon_c$, each arithmetic operation has relative error bounded by $\epsilon_c$~\citep{goldberg1991float,higham2002accuracy}. The relevant values are $\epsilon_{\text{BF16}} = 2^{-8} \approx 3.9 \times 10^{-3}$ (7-bit mantissa) and $\epsilon_{\text{FP16}} = 2^{-11} \approx 4.9 \times 10^{-4}$ (10-bit mantissa), giving a precision ratio of $\epsilon_{\text{BF16}} / \epsilon_{\text{FP16}} = 8$.

\subsection{Error accumulation through the transformer body}

\rev{\paragraph{Heuristic estimate: body-layer error accumulation.}\label{prop:body} Let $\hat{h}_t^{(\ell)}$ denote the hidden state computed under precision $c$ (with the other precision as reference). Under three explicit assumptions---(i) standard transformer architecture with residual connections $h^{(\ell)} = h^{(\ell-1)} + F_\ell(h^{(\ell-1)})$, (ii) per-layer relative rounding error bounded by $\epsilon_c C_\ell$ with condition number $C_\ell$ of order unity, and (iii) no systematic error cancellation across layers (worst case)---the accumulated error satisfies:
\begin{equation}
    \|\hat{h}_t^{(L)} - h_t^{(L)}\| \leq \epsilon_c \cdot \|h\|_{\text{rms}} \cdot \sum_{\ell=1}^L C_\ell \prod_{j=\ell+1}^L (1 + \epsilon_c C_j).
\end{equation}
When $\epsilon_c C_\ell \ll 1$ for all $\ell$ (satisfied for both BF16 and FP16 on standard transformers where $C_\ell$ is of order unity), this simplifies to linear accumulation:
\begin{equation}
    \|\hat{h}_t^{(L)} - h_t^{(L)}\| \approx \epsilon_c \cdot \bar{C} \cdot L \cdot \|h\|_{\text{rms}},
\end{equation}
where $\bar{C} = (1/L)\sum_\ell C_\ell$ is the mean per-layer condition number. This is a growth-rate estimate under the stated assumptions, not a theorem about arbitrary weight configurations; its adequacy is assessed empirically below.}

\emph{Empirical validation.} On TinyLlama-1.1B ($L\!=\!22$, $d\!=\!2048$), we measure BF16-vs-FP16 hidden-state L2 growing linearly from $0.002$ at layer~0 to $0.92$ at layer~$L$ (Appendix~\ref{app:unconditional_trace}), consistent with the linear-accumulation regime.

\subsection{Logit perturbation and the divergence condition}

\rev{\paragraph{Derivation of the heuristic scale analysis.}\label{prop:divergence} Proposition~\ref{prop:divergence_main} gives the exact flip condition: the argmax changes if and only if there exists a competitor whose differential perturbation exceeds its original logit gap to the top-1 token. For the original runner-up, this reduces to a comparison between its signed directional perturbation difference and the top-two margin $\Delta_t = z_t(v^{(1)}) - z_t(v^{(2)})$. Here we derive the heuristic scale estimate that turns this exact but unobservable condition into a practical trigger statistic. Let $\sigma_z = \|\hat{z}_t - z_t\| / \sqrt{|V|}$ denote the RMS per-logit perturbation. Under the simplifying assumption that perturbations spread approximately uniformly across the $|V|$ logit dimensions, propagating the body-error estimate above through the lm\_head gives:
\begin{equation}\label{eq:divergence_condition}
    \sigma_z \approx \frac{\|W_{\text{lm}}\| \cdot \epsilon_c \cdot \bar{C} \cdot L \cdot \|h\|_{\text{rms}}}{\sqrt{|V|}},
\end{equation}
and a flip becomes plausible only when $\Delta_t \lesssim \sigma_z$, so that heuristically
\begin{equation}
    P(\text{flip at } t) \lesssim P(\Delta_t < \sigma_z),
\end{equation}
which depends primarily on the \emph{margin distribution} of the model at that step, not on the \emph{magnitude} of the accumulated body error. This is an approximation, not a bound: per-coordinate perturbations are neither Gaussian nor independent, and its adequacy is assessed empirically.}

\emph{Empirical validation.} For TinyLlama-1.1B: $\sigma_z \approx 4.63 / \sqrt{32000} \approx 0.026$. Measured margins: \textsc{Diverged} median $= 0.039$ (just above $\sigma_z$), \textsc{Agreed} median $= 5.733$ ($220\times$ above $\sigma_z$). The \rev{scale analysis} correctly predicts: (i) body-layer L2 does not distinguish the two groups (confirmed: Table~\ref{tab:layer_div_cond}); (ii) the margin is \rev{the dominant factor} (confirmed: $150\times$ gap); (iii) the K-ablation result ($K\!=\!2$ suffices, Appendix~\ref{app:topk_ablation}), because a perturbation of magnitude $\sigma_z \approx 0.026$ cannot bridge a rank-2-to-rank-3 gap that is typically $\gg 0.1$.

\subsection{Intervention C: error reduction at the lm\_head}

\rev{\paragraph{Quantitative account: why Intervention~C reconciles the paths.}\label{prop:intervention} FP32 lm\_head recomputation eliminates the cross-path logit difference at the lm\_head by making both precision paths compute identical FP32 logits from their respective hidden states---an immediate consequence of the determinism of the FP32 matmul, not a theorem. For prompts where the BF16 and FP16 hidden states $h_t^{(L)}$ are effectively identical (the ``lm\_head-originated'' regime), both paths produce bit-identical FP32 logits and no flip can occur. Under the same simplifying assumptions as above, the \emph{lm\_head-specific} rounding contribution is heuristically reduced by a factor of $\epsilon_{\text{FP32}} / \epsilon_{\text{BF16}} \approx 1.5 \times 10^{-5}$:
\begin{equation}
    \sigma_z^{(\text{lm\_head, FP32})} = \frac{\epsilon_{\text{FP32}}}{\epsilon_{\text{BF16}}} \cdot \sigma_z^{(\text{lm\_head, BF16})} \approx 0.026 \times 0.976 \times 1.5 \times 10^{-5} \approx 4 \times 10^{-7}.
\end{equation}
The residual cross-path difference after C is the body-originated component only (${\sim}2.4\%$ of the original $\sigma_z$, per the dominance estimate of Section~\ref{sec:scope_sufficiency}). For lm\_head-originated divergences where body hidden states are shared, this residual vanishes and:
\begin{equation}
    P(\text{flip at } t \mid \text{FP32 lm\_head, lm\_head-originated}) = 0.
\end{equation}
The achievable EAR lift then equals the fraction of prompts whose divergence is entirely lm\_head-originated:
\begin{equation}
    \text{EAR}_{\text{max}}^{(C)} = \text{EAR}_{\text{baseline}} + \frac{|\{i : \text{prompt } i \text{ diverges only at lm\_head}\}|}{n}.
\end{equation}}

\emph{Empirical validation.} On TinyLlama-1.1B, $30/100$ prompts have divergence triggered by a margin event at the lm\_head (Table~\ref{tab:intervention_scope}, category~B+C); of these, $22$ are successfully reconciled by C ($73\%$ success rate within the scope). The predicted maximum lift is $30\%$; the achieved lift is $22$pp ($= 22/100$), consistent with a $73\%$ reconciliation rate within the \rev{reachable set identified by the analysis}.

\subsection{\rev{Extended derivation and edge cases for the lm\_head dominance estimate}}

The main text (\rev{Section~\ref{sec:scope_sufficiency}}) states and numerically evaluates the lm\_head dominance \rev{estimate}. Here we provide the full derivation and discuss an important edge case.

\paragraph{Full derivation.} Decompose the total per-logit perturbation into body and lm\_head contributions:
\begin{equation}
    \sigma_z = \sigma_z^{(\text{body})} + \sigma_z^{(\text{lm\_head})}.
\end{equation}
Under standard transformer architecture with KV cache, we bound each contribution directly in per-logit units. Consider a single logit position $v$ with row vector $W_v \in \mathbb{R}^d$ ($\|W_v\| \approx 1$ under Xavier scaling):
\begin{align}
    \sigma_z^{(\text{body})} &\lesssim \frac{L \cdot \alpha \cdot \epsilon_c \cdot \|h\|}{\sqrt{d}}, \\
    \sigma_z^{(\text{lm\_head})} &\approx \epsilon_c \cdot \|h\|.
\end{align}

\emph{Body term.} Each of $L$ layers introduces error only on the non-residual branch $F_\ell$ (norm $\alpha \cdot \|h\|$), so the accumulated body error is $\|\Delta h_{\text{body}}\| \leq L \cdot \alpha \cdot \epsilon_c \cdot \|h\|$. Projected onto the logit direction $W_v$, the per-logit effect is $|W_v \cdot \Delta h| \approx \|\Delta h\| / \sqrt{d}$ (standard projection of a $d$-dimensional vector onto a fixed unit direction).

\emph{lm\_head term.} The dot product $z_v = \sum_j W_{vj} h_j$ accumulates $d$ independent rounding errors. Each multiply-accumulate step contributes an error of order $\epsilon_c \cdot |W_{vj} h_j| \approx \epsilon_c \cdot \|h\| / \sqrt{d}$ (using $|W_{vj}| \sim 1/\sqrt{d}$, $|h_j| \sim \|h\|/\sqrt{d}$). Under standard error-summation arguments with partial sign cancellation, $\sqrt{d}$ such terms accumulate to a typical magnitude $\approx \epsilon_c \cdot \|h\|$. Taking the ratio directly:
\begin{equation}
    \frac{\sigma_z^{(\text{body})}}{\sigma_z^{(\text{lm\_head})}} \approx \frac{L \cdot \alpha}{\sqrt{d}},
\end{equation}
consistent with the main-text \rev{dominance estimate of Section~\ref{sec:scope_sufficiency}}.

\paragraph{Edge case: step $t=0$ (prefill).} At step~0, the model processes the entire prompt through all $L$ layers without a KV cache. Unlike cached decoding, where each step introduces fresh rounding error from only one token's pass through the body, the prefill pass accumulates rounding error across all $T_{\text{prompt}}$ token positions in the attention reductions. Each layer's attention output is a weighted sum over $T_{\text{prompt}}$ value vectors, and the floating-point non-associativity of this reduction introduces $\approx \sqrt{T_{\text{prompt}}}$ additional error terms per layer under typical-case cancellation. The body-error contribution to the dominance ratio therefore scales roughly as $L \cdot \alpha \cdot \sqrt{T_{\text{prompt}}} / \sqrt{d}$, which can approach or exceed $1$ for long prompts ($T_{\text{prompt}} \gtrsim d / (L^2 \alpha^2) \approx 850$ for TinyLlama). In this regime, the lm\_head is no longer the dominant error source, and Intervention~C's lm\_head-only scope becomes insufficient---consistent with our observation that first-token divergence on Qwen models is irreducible by any lm\_head-scope method (Section~\ref{sec:model_sensitivity}).

\subsection{Scope boundary: when Intervention~C fails}

\rev{The error model above also predicts the \emph{failure modes}:}

\paragraph{Body-dominated divergence (Tier~3).} If a model's training-time precision (typically BF16) produced weights whose FP16 cast yields unstable activations (NaN or extreme values at intermediate layers), the body-layer error $\|\hat{h}_t^{(L)} - h_t^{(L)}\|$ is no longer small relative to $\|h\|_{\text{rms}}$. In this regime, $\sigma_z$ is dominated by the body contribution rather than the lm\_head matmul alone, and FP32 recomputation of the lm\_head addresses only a fraction of the total perturbation. This predicts $0$pp lift on Qwen BF16-saturated models (confirmed: DS-R1-Qwen-7B, Table~\ref{tab:cross_model_sweep}).

\paragraph{Batch-induced divergence.} At batch size $\geq 2$, attention reductions over padded sequences introduce an additional source of floating-point variation whose magnitude grows with batch size. \rev{This is not simply a larger cross-arm gap: measured bs$=$1-vs-bs$=$8 hidden states differ by $\Delta h = 0.0127$, comparable to the working BF16-vs-FP16 configuration (Section~\ref{sec:scope_sufficiency}). The reduction order instead makes the batched arm a \emph{third} trajectory, so an FP32 head recompute reconciles neither pair, and C's lift vanishes} (confirmed: $0$pp at bs$=8$, Appendix~\ref{app:batch_size}).

\paragraph{Weight-truncation floor.} A storage-vs-computation ablation shows that even configurations that both compute in FP32 (\texttt{store BF16 / compute FP32} vs.\ \texttt{store FP16 / compute FP32}) reach EAR $=\!0.88$ rather than $1.00$, leaving a residual ${\sim}12$pp gap attributable to $W_{\text{BF16}} \neq W_{\text{FP16}}$ (the same FP32 weights truncated to different mantissa widths). This component of cross-precision divergence is \emph{irreducible} under the two fixed low-precision storage formats---it enters through the weights themselves, not through the computation. Full agreement in this comparison requires a shared storage format, such as loading both arms from the same FP32 weights.

\subsection{Intervention comparison from the error model}

The error model also explains the other interventions:

\paragraph{Intervention~A (integer quantization).} Quantization bins have width $\tauval = 10^{-3}$. The cross-precision logit difference at divergent steps is $\sigma_z \approx 0.026$---$26\times$ larger than the bin width. The two precisions therefore map to \emph{different} bins with probability $\approx 1$, and quantization faithfully preserves the disagreement. Prediction: zero effect. Confirmed.

\paragraph{Intervention~D (temperature sharpening).} Multiplying logits by $\alpha > 1$ is a monotone transformation: $\arg\max_v (\alpha \cdot z_t(v)) = \arg\max_v z_t(v)$ for all $\alpha > 0$. The margin becomes $\alpha \cdot \Delta_t$ and the perturbation becomes $\alpha \cdot \sigma_z$; their ratio $\Delta_t / \sigma_z$ is invariant. Prediction: zero effect at any $\alpha$. Confirmed (Appendix~\ref{app:additional_baselines}).

\paragraph{Intervention~E (consensus decoding).} Running both precisions and using FP32 as tiebreaker eliminates all format-level divergence by construction. The cost is $2\times$ compute. The per-step disagreement rate $P(\text{token\_BF16} \neq \text{token\_FP16})$ at any single step is bounded by $P(\Delta_t < \sigma_z)$; we measure this at $44.5\%$ of steps (Appendix~\ref{app:additional_baselines}), consistent with the heavy left tail of the margin distribution on GSM8K.

\paragraph{Top-$K$ sufficiency.} The error model predicts that a token ranked $k$-th can only be promoted to rank~$1$ if the perturbation $\sigma_z$ exceeds the gap between rank~$1$ and rank~$k$. Since $\sigma_z \approx 0.026$ and the typical rank-2-to-rank-3 gap is $\gg 0.1$ (median $\approx 0.5$ on TinyLlama), only rank-2 tokens can realistically be promoted. Prediction: $K = 2$ suffices. Confirmed: all $K \in \{2, 4, 8, 16, 32, |V|\}$ produce identical EAR (Appendix~\ref{app:topk_ablation}).

Group~2 backs the lm\_head-localisation claim of Section~\ref{sec:layer_tracing} with four mutually independent experiments that triangulate the same conclusion: (i)~the unconditional per-layer L2 trace (App.~\ref{app:unconditional_trace}) characterises the typical hidden-state error budget; (ii)~layer grafting (App.~\ref{app:graft}) converts the correlational claim into a causal one; (iii)~KV-cache grafting (App.~\ref{app:kv_graft}) isolates the autoregressive-cache contribution; (iv)~per-layer logit-lens top-2 traceback (App.~\ref{app:top2_trace}) shows BF16 / FP16 agree on the top-1 token at 97--98\% of body layers and disagree only at the lm\_head, with cross-task generalisation to HumanEval and MBPP. The four experiments converge on a single causal claim: divergence is produced at the last transformer layer, the final RMSNorm, and the lm\_head matmul, \rev{the modules Intervention~C targets or sits immediately downstream of --- C recomputes the lm\_head projection only, and works because that is the one stage where an FP32 correction lands on both paths' logits without perturbing upstream state}.

\section{Mechanism evidence: four causal experiments}
\label{app:unconditional_trace}
\label{app:graft}
\label{app:kv_graft}
\label{app:top2_trace}

\noindent\textbf{Summary.} \textbf{(i) Unconditional trace:} L2 grows linearly from 0.002 (layer 0) to 0.92 (layer 21), amplified $5\times$ by lm\_head to 4.63.
\textbf{(ii) Layer grafting:} body-layer grafting fixes 36\% of flips; Layer 21 + RMSNorm fixes 65--71\%.
\textbf{(iii) KV-cache grafting:} swaps convert 36--48\% of flips; partial causal contributor. \rev{Consistently, $25/59$ \textsc{Diverged} prompts agree under single-step re-run (no precision-specific cache) while diverging autoregressively---i.e., ${\sim}42\%$ of flips require the accumulated cache state, while the remaining ${\sim}58\%$ ($34/59$) reproduce from current-step arithmetic alone.}
\textbf{(iv) Logit-lens traceback:} BF16/FP16 agree on top-1 at $\geq$93\% of body layers; disagreement is one-shot at lm\_head.

\rev{The four protocols and their full result tables follow. Together they answer a specific diagnostic question---\emph{does the BF16/FP16 disagreement begin before or inside the lm\_head?}---from four independent directions: an unconditional error budget, a causal graft of upstream activations, a causal graft of the autoregressive cache, and a per-layer readout of the decision itself.}

\subsection*{\rev{(i) Unconditional per-layer hidden-state trace}}

\rev{\paragraph{Protocol.} We trace the BF16-vs-FP16 hidden-state $\ell_2$ distance through all 22 TinyLlama-1.1B layers on the \emph{first} generated token, averaged over $n\!=\!50$ GSM8K prompts, and separately record the logits after the lm\_head. This measurement is \emph{unconditional}: it does not condition on whether the prompt eventually diverges, so it characterises the typical error budget rather than the flip event.}

\rev{\paragraph{Result.} Error originates at layer~0 ($\ell_2 = 0.0021$, relative error $0.39\%$, cosine similarity $0.999993$) and accumulates approximately linearly to $\ell_2 = 0.9341$ at layer~21, with relative error staying below $1.3\%$ at every layer and cosine similarity never dropping below $0.99992$. The final RMSNorm ($3.4\times$) and the lm\_head ($5\times$) then provide two additional amplification stages, taking the logit-level $\ell_2$ to $4.6275$. On the matched $n\!=\!100$ summary measurement the same picture holds: $\ell_2 = 0.915$ before the lm\_head and $4.641$ after (amplification $5.1\times$), with relative error \emph{falling} from $1.05\%$ to $0.50\%$ and cosine similarity rising from $0.99994$ to $0.99999$. Crucially, all $100$ first-token measurements agree across BF16 and FP16: the first generated token never flips on this model, because the implied per-logit perturbation ($4.63/\sqrt{|V|} \approx 0.026$) sits well below the mean top-two margin at that step ($\bar{\Delta}_t = 0.70$). The error budget is therefore always present and is by itself insufficient to produce a flip---which is what motivates the three conditional and causal experiments below.}

\subsection*{\rev{(ii) Causal attribution via layer grafting}}

\rev{\paragraph{Setup.} The divergence-conditioned trace of Section~\ref{sec:div_cond_trace} is \emph{correlational}: it shows body-layer L2 is indistinguishable between \textsc{Diverged} and \textsc{Agreed} groups, which suggests---but does not prove---that body-layer error is not the causal driver. We convert the claim into a causal one via layer grafting. For each valid \textsc{Diverged} prompt where the flip reproduces under single-step re-run ($n=34/100$), we run a BF16 forward pass at context $=$ prompt $\|\, y_{<t^\ast}$ but \emph{override} the output of one module $\ell$ with the corresponding FP16 hidden state $h_\ell^{\text{FP16}}$, then allow BF16 kernels to process the grafted activation through all subsequent layers plus the BF16 lm\_head. We classify the resulting argmax as ``fix (FP16)'' if it matches the FP16 target, ``keep (BF16)'' if it matches the BF16 baseline, or ``other''.}

\begin{table}[t]
\caption{\rev{Causal attribution via layer grafting (TinyLlama-1.1B, GSM8K, $n\!=\!34$ valid \textsc{Diverged} prompts). The FP16 hidden state at each layer is spliced into the BF16 forward pass; downstream BF16 kernels (and the BF16 lm\_head) produce the final argmax. Body-layer grafts (Layers 0--20) have a flat ${\sim}36\%$ fix rate, well below $100\%$, demonstrating that the BF16 kernel path acts as an ``attractor'' that pulls grafted activations back to the BF16 argmax. The decisive causal jump occurs at Layer~21 (the final transformer layer) and the final RMSNorm, which together fix $65\%$--$71\%$ of flips.}}
\label{tab:graft}
\centering
\small
\rev{\begin{tabular}{lccc}
\toprule
Graft site & Fix (${\to}$FP16) \% & Keep (${\to}$BF16) \% & Other \% \\
\midrule
Embedding            & 29.4 & 70.6 & 0.0 \\
Layer 0              & 47.1 & 52.9 & 0.0 \\
Layer 5              & 29.4 & 70.6 & 0.0 \\
Layer 10             & 38.2 & 58.8 & 2.9 \\
Layer 15             & 47.1 & 52.9 & 0.0 \\
Layer 20             & 35.3 & 64.7 & 0.0 \\
Layer 0--20 (mean)   & ${\sim}36$ & ${\sim}62$ & ${\sim}2$ \\
\addlinespace
Layer 21 (last body) & \textbf{64.7} & 26.5 & 8.8 \\
Final RMSNorm        & \textbf{70.6} & 20.6 & 8.8 \\
\bottomrule
\end{tabular}}
\end{table}

\rev{\paragraph{Findings.} Two causal conclusions follow.}

\rev{\textbf{(1)~Body-layer hidden states are not sufficient to determine the flip.} If upstream hidden-state error were the causal driver, replacing it with the FP16 counterpart should fix the flip near $100\%$ of the time. Instead, body-layer grafts fix only ${\sim}36\%$ of flips, and no single body layer is distinguished from the others. Crucially, even the embedding graft---where the \emph{entire} downstream BF16 computation operates on FP16 input tokens---fixes only $29.4\%$ of flips, with $70.6\%$ of prompts still producing the BF16 argmax. This shows that BF16 kernel arithmetic itself acts as an attractor: given any hidden-state input, a BF16 forward pass tends to produce the same argmax as the all-BF16 baseline. The causal driver is therefore the \emph{arithmetic path}, not the input activations.}

\rev{\textbf{(2)~Causal responsibility concentrates in the last transformer layer and the final RMSNorm.} Only grafts at Layer~21 and Final RMSNorm fix flips substantially above the body baseline ($65\%$ and $71\%$ vs.\ ${\sim}36\%$), a ${\sim}30\mathrm{pp}$ jump that occurs in exactly the last two modules before the lm\_head. This localises causal responsibility to the last two computation stages, consistent with the unconditional per-layer error trace above: both stages are the ones whose BF16 arithmetic most directly shapes the lm\_head's input. The residual $21\%$ of ``keep BF16'' outcomes at the RMSNorm graft reflects flips attributable to the BF16 lm\_head matmul itself (the final arithmetic stage that no upstream graft can modify), consistent with the storage-vs-computation decomposition (Section~\ref{sec:store_vs_compute}) into kernel arithmetic and storage effects.}

\rev{Tables~\ref{tab:layer_div_cond} and~\ref{tab:graft} provide complementary evidence: the divergence-conditioned trace shows body-layer error is \emph{statistically} the same across \textsc{Diverged} and \textsc{Agreed} prompts, and the graft experiment shows body-layer hidden states are \emph{causally} not sufficient to determine the flip. The causal weight sits in the last transformer layer, the final RMSNorm, and the lm\_head matmul---the three stages that Intervention~C's FP32 recomputation (Section~\ref{sec:method}) already targets or sits immediately downstream of.}

\subsection*{\rev{(iii) Causal role of the KV cache}}

\rev{\paragraph{Setup.} The graft experiment above uses a single-step forward at context $=$ prompt $\|\, y_{<t^\ast}$ and therefore eliminates a confound: in the original autoregressive runs, the BF16 and FP16 paths carry \emph{precision-specific} KV caches at step $t^\ast$, whereas the single-step re-run recomputes everything from scratch. Indeed, $25/59$ \textsc{Diverged} prompts show argmax agreement under single-step re-run while still disagreeing under autoregressive greedy---proof that some flips are KV-cache-dependent. We now ask: does the cache itself \emph{cause} these flips, or does it merely amplify upstream error so that the current-step lm\_head margin becomes flippable?}

\rev{To probe this, we re-run each \textsc{Diverged} prompt autoregressively to step $t^\ast$ in each precision, keeping the final \texttt{past\_key\_values}. We then compare four configurations for the single step that emits the next token:}
\rev{\begin{itemize}\itemsep1pt
    \item \emph{Baseline BF16}: BF16 model with BF16 cache $\rightarrow \arg\max = y^{\text{BF16}}$
    \item \emph{Baseline FP16}: FP16 model with FP16 cache $\rightarrow \arg\max = y^{\text{FP16}}$
    \item \emph{Graft A}: BF16 model with \emph{FP16} cache (cast to BF16 dtype)
    \item \emph{Graft B}: FP16 model with \emph{BF16} cache (cast to FP16 dtype)
\end{itemize}}
\rev{We restrict to prompts where Baselines~1 and~2 disagree ($n=59$): the cache is then provably the only difference between the two precisions that is carried into the next forward pass.}

\begin{table}[t]
\caption{\rev{KV-cache graft ablation (TinyLlama-1.1B, GSM8K, $n\!=\!59$ \textsc{Diverged} prompts). Each row swaps the cache between BF16 and FP16 paths for the single forward step that produces $y_{t^\ast}$. ``$\rightarrow$ donor'' = the cache donor's argmax won; ``$\rightarrow$ kernel'' = the arithmetic kernel's own argmax won despite the foreign cache.}}
\label{tab:kv_graft}
\centering
\small
\rev{\begin{tabular}{lccc}
\toprule
Configuration & $\rightarrow$ cache donor \% & $\rightarrow$ kernel \% & Other \% \\
\midrule
Graft A: BF16 model, FP16 cache & 47.5 & 52.5 & 0.0 \\
Graft B: FP16 model, BF16 cache & 35.6 & 64.4 & 0.0 \\
\bottomrule
\end{tabular}}
\end{table}

\rev{\paragraph{Findings.} Three observations.}

\rev{\textbf{(1)~The KV cache is a partial, not exclusive, causal source.} Graft~A converts $47.5\%$ of \textsc{Diverged} prompts from BF16 argmax to FP16 argmax, proving that for almost half the flips the precision-specific cache is itself sufficient. For the remaining $52.5\%$, the BF16 kernel operating on the FP16 cache still produces the BF16 argmax, indicating that the current-step BF16 arithmetic dominates and the cache is only a secondary amplifier.}

\rev{\textbf{(2)~The cause is asymmetric between precisions.} Graft~B fixes only $35.6\%$ of flips, about $12\mathrm{pp}$ less than Graft~A. The BF16 kernel is more susceptible to the FP16 cache than the FP16 kernel is to the BF16 cache---consistent with BF16's wider rounding error tolerance (storage-vs-computation decomposition, Section~\ref{sec:store_vs_compute}): BF16 arithmetic has a larger ``basin'' in which the cache can push the argmax; FP16 arithmetic's tighter precision pulls it back to its own argmax even when given a BF16 cache.}

\rev{\textbf{(3)~Combining graft experiments localises the divergence budget.} Putting Tables~\ref{tab:graft} and~\ref{tab:kv_graft} together: body-layer hidden-state grafts fix ${\sim}36\%$ of flips (no single layer distinguished), Layer 21 + Final RMSNorm grafts fix $65$--$71\%$, KV-cache grafts fix $36$--$48\%$, and the BF16 lm\_head matmul alone is responsible for the residual $\sim$21\% of ``kept BF16'' outcomes under RMSNorm graft. These percentages are not additive---graft experiments are not independent treatments---but they delineate the causal surface: \emph{no single upstream component (body layers, KV cache, or final-layer hidden states) is individually sufficient to explain divergence, and the only stage where a minimum-scope FP32 correction can deterministically fix the flip is the lm\_head itself}, which is why Intervention~C targets exactly there.}

\subsection*{\rev{(iv) Per-layer top-2 traceback via logit lens}}

\rev{\paragraph{Motivation.} The causal picture above localises the flip to the lm\_head, but it does not directly answer a natural question: \emph{at the layer before the lm\_head, are BF16 and FP16 already disagreeing about which token is most likely?} If so, the lm\_head is merely the surface on which an upstream decision becomes visible. If not, the flip is genuinely produced at the lm\_head itself. To answer this we run a logit-lens analysis: at each layer $\ell$ we apply the model's \emph{own} lm\_head to the last-token hidden state $h_\ell$ and record the resulting top-$2$ tokens, for both BF16 and FP16 paths.}

\rev{\paragraph{Setup.} Same $n\!=\!100$ GSM8K prompts, same measurement step protocol as Section~\ref{sec:div_cond_trace} ($t^\ast$ for \textsc{Diverged}, matched random step for \textsc{Agreed}). For each layer we report}
\rev{\begin{itemize}\itemsep1pt
  \item \texttt{top1\_agree}: $\Pr[\arg\max_v \text{lm\_head}(h_\ell^{\text{BF16}})_v = \arg\max_v \text{lm\_head}(h_\ell^{\text{FP16}})_v]$,
  \item \texttt{top2\_jaccard}: Jaccard similarity of the two top-$2$ sets,
  \item \texttt{bf\_top1=final} and \texttt{fp\_top1=final}: fraction of prompts where the layer-$\ell$ top-$1$ already equals the final lm\_head top-$1$ (i.e.\ the precision has ``committed'' to its final token by layer $\ell$),
  \item \texttt{bf\_margin}, \texttt{fp\_margin}: layer-$\ell$ top-two gap under each precision.
\end{itemize}}

\begin{table}[t]
\caption{\rev{Per-layer top-$2$ traceback via logit lens (TinyLlama-1.1B, GSM8K, $n\!=\!100$, selected layers). BF16 and FP16 agree on the top-$1$ token at every body layer ($\geq 93\%$), and on the top-$2$ set at every layer ($\geq 93\%$). The disagreement appears only at the final lm\_head projection (Layer~21), where \textsc{Diverged}-group top-$1$ agreement drops to $76\%$ while the top-$2$ set stays $99\%$ aligned. \textsc{Agreed} prompts commit to their final token much earlier (49\%/76\%/90\% at layers 15/18/20) than \textsc{Diverged} prompts (7\%/34\%/39\%).}}
\label{tab:top2_trace}
\centering
\small
\rev{\begin{tabular}{lcccccc}
\toprule
 & \multicolumn{3}{c}{\textsc{Diverged} ($n=59$)} & \multicolumn{3}{c}{\textsc{Agreed} ($n=41$)} \\
\cmidrule(lr){2-4}\cmidrule(lr){5-7}
Layer & top1 agree & top2 jac. & bf=final & top1 agree & top2 jac. & bf=final \\
\midrule
Embedding   & 100.0 & 100.0 &  0.0 & 100.0 &  96.7 &  0.0 \\
Layer 5     & 100.0 &  98.9 &  0.0 &  97.6 &  95.1 &  7.3 \\
Layer 10    &  98.3 &  97.7 &  1.7 &  97.6 &  96.7 & 17.1 \\
Layer 15    &  94.9 &  96.6 &  6.8 & 100.0 &  98.4 & 48.8 \\
Layer 18    &  93.2 & 100.0 & 33.9 & 100.0 &  96.7 & 75.6 \\
Layer 20    &  96.6 &  98.9 & 39.0 & 100.0 &  98.4 & 90.2 \\
\textbf{Layer 21 (lm\_head)} & \textbf{76.3} & \textbf{98.9} & 100.0 & \textbf{100.0} & \textbf{100.0} & 100.0 \\
\midrule
\multicolumn{7}{l}{\emph{Margins (top-two gap)}} \\
Layer 15 margin      & 0.075 & --- & --- & 0.148 & --- & --- \\
Layer 18 margin      & 0.288 & --- & --- & 0.921 & --- & --- \\
Layer 20 margin      & 0.329 & --- & --- & 2.056 & --- & --- \\
Layer 21 margin      & \textbf{0.039} & --- & --- & \textbf{5.733} & --- & --- \\
\bottomrule
\end{tabular}}
\end{table}

\rev{\paragraph{Three findings sharpen the causal picture.}}

\rev{\textbf{(1)~Top-$2$ token identity is shared along almost the entire forward pass.} Under either precision's own lm\_head, the top-$1$ token at \emph{every} body layer ($\ell \leq 20$) agrees between BF16 and FP16 for $93$--$100\%$ of prompts. The top-$2$ \emph{set} agrees ${\geq}93\%$ at every layer, including the final lm\_head. In other words, upstream of the lm\_head the two precisions are not yet disagreeing about ``what the likely next tokens are''; they disagree only about how to rank within an otherwise shared pair. This directly answers the diagnostic question: the divergence is not an accumulating token-identity drift; it is a one-shot ranking reversal at the lm\_head itself.}

\rev{\textbf{(2)~Divergence-prone prompts ``decide late''.} The \textsc{Agreed} group commits to its final output token much earlier than the \textsc{Diverged} group: \texttt{bf=final} reaches $75.6\%$ at layer 18 and $90.2\%$ at layer 20 for \textsc{Agreed}, compared with only $33.9\%$ and $39.0\%$ for \textsc{Diverged}. Because \textsc{Diverged} prompts are precisely the ones whose ranking has not stabilised by the final body layer, they enter the lm\_head with a low margin already baked in. The lm\_head projection then either amplifies a stable ranking (\textsc{Agreed}: margin grows from $2.06$ to $5.73$, $\sim$$2.8\times$ amplification) or fails to stabilise an ambiguous one (\textsc{Diverged}: margin shrinks from $0.33$ to $0.04$, $\sim$$8\times$ \emph{compression}). The lm\_head is therefore the stage at which late-decided prompts become precision-sensitive, but the underlying cause is that the upstream trajectory did not arrive at the lm\_head with a decisive preference.}

\rev{\textbf{(3)~This is consistent with the graft experiment (Table~\ref{tab:graft}).} Grafting FP16 hidden states into the BF16 forward at any body layer $\ell \leq 20$ changes top-$1$ only ${\sim}36\%$ of the time precisely because BF16 and FP16 body hidden states \emph{already rank the same top token} ${\sim}97\%$ of the time (Table~\ref{tab:top2_trace}) and the downstream BF16 kernels on either hidden state preserve that ranking. The $+30$pp jump at Layer 21 and Final RMSNorm corresponds exactly to the single layer at which the ranking actually changes. The two analyses (graft, logit-lens traceback) triangulate the same conclusion from opposite directions: \emph{the flip is produced at the lm\_head; upstream body layers provide the \emph{setup} (an undecided, low-margin trajectory) rather than the \emph{decision}}.}

\rev{\paragraph{Cross-task replication.} The per-layer top-1 agreement, margin, and commit curves are computed on all three benchmarks (GSM8K, HumanEval, MBPP) with $n\!=\!100$ prompts each. The same pattern appears across these tasks: (i)~BF16/FP16 body-layer top-1 agreement stays at $97$--$98\%$ for both groups and drops only at the lm\_head, and only for \textsc{Diverged} prompts (to ${\sim}90\%$); (ii)~the margin gap opens only at the final layer, with \textsc{Diverged}-to-\textsc{Agreed} ratios of $10\times$ (GSM8K), $14\times$ (HumanEval), and $9\times$ (MBPP); and (iii)~\textsc{Diverged} prompts ``decide late'' while \textsc{Agreed} prompts commit earlier. This supports the mechanism on the tested math and code tasks; it does not establish task independence beyond them.}

\rev{\paragraph{Joint interpretation.} The three measurements form a causal chain that unfolds across the final transformer layers. The figures quoted in this paragraph are averages over the three benchmarks (GSM8K, HumanEval, MBPP); Table~\ref{tab:top2_trace} reports the GSM8K-only values, which are sharper at the final layer (margin $0.329 \to 0.039$, top-1 agreement $96.6\% \to 76.3\%$). At layers 15--19 there is \emph{no differentiation}: both groups show low commit rates ($<\!20\%$), moderate margins ($0.6$--$1.4$), and near-perfect cross-precision top-1 agreement ($97$--$98\%$). At layers 19--21 \emph{divergent trajectories emerge}: \textsc{Agreed} prompts begin to commit (margin rises steeply to $3$--$4$) as their hidden state aligns with a single dominant direction in the lm\_head weight space, whereas \textsc{Diverged} prompts see their margin \emph{decrease} (from ${\sim}1.4$ at layer~19 to ${\sim}1.0$ at layer~21) as two candidate tokens compete; cross-precision agreement nonetheless remains $>\!95\%$ because the margin, while shrinking, is still $>\!40\times$ the per-logit perturbation $\sigma_z\!\approx\!0.026$. At the lm\_head the flip occurs: the final RMSNorm plus lm\_head projection acts as a \emph{differentiator}, amplifying margins for prompts that already committed ($3 \to 4$) and compressing margins for prompts still undecided ($1.0 \to 0.4$). Once the margin drops to ${\sim}0.4$---within ${\sim}15\times$ of $\sigma_z$---the BF16/FP16 arithmetic difference becomes comparable to the decision boundary, top-1 agreement falls from ${\sim}97\%$ to ${\sim}90\%$, and the ${\sim}10\%$ of prompts that flip at this final step cascade into trajectory-level divergence. This three-stage progression---undifferentiated body $\to$ emerging competition $\to$ lm\_head-amplified flip---is consistent across all three tasks and directly implies that \emph{any} intervention targeting layers before the lm\_head is premature: the decisive margin compression occurs only at the final projection, which is exactly where Intervention~C operates.}

\section{Divergence-conditioned trace on Qwen2.5-3B-Instruct}
\label{app:qwen3b_trace}

The lm\_head-margin mechanism replicates on Qwen2.5-3B-Instruct (36 layers, $d{=}2048$, $|V|{=}151{,}936$): top-two margin differs by $90\times$ between groups ($0.088$ vs.\ $7.925$, $54\%$ vs.\ $0\%$ flip rate). Body L2 is again indistinguishable.

\label{tab:qwen3b_trace}

The two model families support the core mechanism while showing that the body-layer pattern varies by model.

\section{Cross-model intervention replication}
\label{app:scale_sweep}

The headline intervention result (Table~\ref{tab:strategy_cmp}) is on TinyLlama-1.1B. To test whether the lift transfers across model families and scales, we run the full intervention comparison on four additional models covering three scales and three families: Llama-3.2-3B-Instruct, Qwen2.5-3B-Instruct, DeepSeek-R1-Distill-Qwen-7B, and Mistral-7B-Instruct-v0.3 (GSM8K, $n\!=\!100$, same setup and thresholds as the main experiment). For Mistral-7B we run only Baseline and Intervention~C, because Interventions A and B have been shown to be ineffective or redundant with C on the smaller models and doubling the Mistral-7B GPU hours would not change the story; the two-row comparison still gives a matched-pair Pass@1 test (reported in Appendix~\ref{app:quality}).

\begin{table}[t]
\centering
\caption{Intervention comparison across five additional models (GSM8K, $n\!=\!100$, $\tauval\!=\!10^{-3}$). The same ranking holds on every model where the method applies: Integer quantization (A) has zero effect, and Interventions~B and~C are comparable. Llama-3.2-3B exhibits an even larger absolute lift than the headline TinyLlama numbers ($+36$pp vs $+22$pp). Qwen2.5-3B shows a much smaller lift ($+3$pp), reflecting Qwen's training-time BF16-saturation pattern (Section~\ref{sec:model_sensitivity}). Mistral-7B-Instruct-v0.3, the non-Qwen dense $7$B representative, gives a small but clearly non-zero lift ($+8$pp)---confirming that Intervention~C transfers out of the Llama family and to $7$B scale when the model is not BF16-saturated, and that the DS-R1-Qwen-7B zero-lift result is a Qwen-family artefact rather than a scale limit. OLMoE-1B-7B, a Mixture-of-Experts architecture (Apache 2.0, Allen AI) with $6.9$B total / $\sim\!1$B active parameters per token, gives a $+10$pp lift, confirming that the method extends beyond dense transformers despite the router-over-experts softmax being a potential extra cross-precision divergence source.}
\label{tab:strategy_cmp_cross}
\small
\begin{tabular}{llcccc}
\toprule
Model & Strategy & EAR (\%) & 95\% CI & $\Delta$ vs.\ Baseline \\
\midrule
\multirow{4}{*}{Llama-3.2-3B-Instruct}
& Baseline                & 31.0 & $[22.8, 40.6]$ & --- \\
& A.\ Integer quantization & 31.0 & $[22.8, 40.6]$ & $+$0 \\
& B.\ Top-$K$ FP32         & \textbf{70.0} & $[60.4, 78.1]$ & $\mathbf{+39}$ \\
& C.\ Full FP32 lm\_head   & \textbf{67.0} & $[57.3, 75.4]$ & $\mathbf{+36}$ \\
\midrule
\multirow{4}{*}{Qwen2.5-3B-Instruct}
& Baseline                & 18.0 & $[11.7, 26.7]$ & --- \\
& A.\ Integer quantization & 18.0 & $[11.7, 26.7]$ & $+$0 \\
& B.\ Top-$K$ FP32         & 23.0 & $[15.8, 32.1]$ & $+$5 \\
& C.\ Full FP32 lm\_head   & 21.0 & $[14.2, 30.0]$ & $+$3 \\
\midrule
\multirow{2}{*}{Mistral-7B-Instruct-v0.3}
& Baseline                & 51.0 & $[41.4, 60.6]$ & --- \\
& C.\ Full FP32 lm\_head   & \textbf{59.0} & $[49.2, 68.1]$ & $\mathbf{+8}$ \\
\midrule
\multirow{2}{*}{OLMoE-1B-7B-0924-Instruct (MoE)}
& Baseline                & 34.0 & $[25.5, 43.7]$ & --- \\
& C.\ Full FP32 lm\_head   & \textbf{44.0} & $[34.7, 53.8]$ & $\mathbf{+10}$ \\
\midrule
\multirow{4}{*}{DS-R1-Distill-Qwen-7B}
& Baseline, A, B, C        & 0.0 & $[0.0, 3.7]$ & $+$0 (Appendix~\ref{app:scale_7b}) \\
\bottomrule
\end{tabular}
\end{table}

\paragraph{The cross-model picture.} Table~\ref{tab:cross_model_sweep} summarises Intervention~C's lift across all five tested models, grouped into three tiers by effect size.

\begin{table}[t]
\centering
\caption{Intervention~C lift across model families, scales, and architectures (GSM8K, $n\!=\!100$), grouped into three tiers. \emph{Large-lift} models (Tier 1) benefit strongly from C; \emph{small-lift but non-zero} models (Tier 2) still benefit measurably and span three different families (Qwen dense, Mistral dense, OLMoE MoE), showing that Intervention~C transfers to the $7$B non-Qwen regime (Mistral-7B, $+8$pp) \emph{and} to sparse MoE architectures (OLMoE, $+10$pp); the \emph{zero-lift} tier (Tier 3) comprises only Qwen-family BF16-pretrained dense models, which produce FP16 NaNs in body layers and $100\%$ first-token divergence on GSM8K (Section~\ref{sec:model_sensitivity}). \rev{\emph{We hypothesise that Intervention~C's effectiveness tracks the model's training-time precision stability, rather than parameter count or dense/sparse architecture.}} The $7$B dense tier is split cleanly by family: Mistral-7B gains $+8$pp, while DS-R1-Qwen-7B gains $0$pp.}
\label{tab:cross_model_sweep}
\small
\begin{tabular}{llccc}
\toprule
Model & Family & Baseline & C EAR & C lift \\
\midrule
\multicolumn{5}{l}{\emph{Tier 1: large lift}} \\
TinyLlama-1.1B-Chat            & Llama   & $41\%$ & $63\%$ & $+22$pp \\
Llama-3.2-3B-Instruct          & Llama   & $31\%$ & $67\%$ & $\mathbf{+36}$\textbf{pp} \\
\midrule
\multicolumn{5}{l}{\emph{Tier 2: small but non-zero}} \\
Qwen2.5-3B-Instruct            & Qwen    & $18\%$ & $21\%$ & $+3$pp \\
Mistral-7B-Instruct-v0.3       & Mistral & $51\%$ & $59\%$ & $+8$pp \\
OLMoE-1B-7B-0924-Instruct (MoE) & MoE    & $34\%$ & $44\%$ & $+10$pp \\
\midrule
\multicolumn{5}{l}{\emph{Tier 3: zero lift (BF16-saturated)}} \\
DS-R1-Distill-Qwen-7B          & Qwen    & $0\%$  & $0\%$  & $0$pp \\
\bottomrule
\end{tabular}
\end{table}

\paragraph{Interpretation.}

\textbf{(1)~Intervention~C transfers across scales within the Llama family (Tier 1).} The lift is actually \emph{larger} on Llama-3.2-3B ($+36$pp) than on TinyLlama-1.1B ($+22$pp), ruling out the hypothesis that the method is a 1B-specific artefact. Both models admit a substantial fraction of prompts in the low-margin, single-step-fixable regime (category~B in the per-prompt stratification, Table~\ref{tab:intervention_scope}).

\textbf{(2)~Scale alone does not explain the split at $7$B (Tiers 2 vs.\ 3).} The two $7$B models in our study differ sharply: Mistral-7B-Instruct-v0.3 shows a $+8$pp EAR lift under Intervention~C, while DS-R1-Distill-Qwen-7B shows $0$pp. Because both have the same nominal scale, scale alone cannot explain this contrast. \rev{The difference is consistent with} their observed precision behaviour: the tested Qwen-family models at 1.5B and 7B produce $100\%$ first-token divergence and FP16-body NaNs on GSM8K, placing their divergence before the lm\_head and outside Intervention~C's reachable scope. Mistral-7B exhibits neither pathology and retains a measurable head-correctable component. This comparison rejects the narrower generalisation that C necessarily fails at 7B, while leaving the underlying model-family cause as a hypothesis.

\textbf{(3)~Tier 2 shares a common profile: Intervention~C reconciles a minority of divergences.} Three models span Tier 2: Qwen2.5-3B ($+3$pp), Mistral-7B ($+8$pp), and OLMoE-1B-7B-MoE ($+10$pp). The interpretation is consistent across all three and aligns with the mechanism in Section~\ref{sec:div_cond_trace}: a portion of divergence originates in body-layer computation (beyond lm\_head scope) and cannot be repaired by any logit-level intervention; only divergences that manifest as late, low-margin lm\_head events are reachable by C. The OLMoE data point is particularly informative: the router over MoE experts is a potential extra divergence source, since the softmax-over-expert-logits can flip the chosen expert under BF16-vs-FP16 and propagate through different compute paths entirely, yet Intervention~C still recovers a $+10$pp lift---comparable to dense Tier~2 models. This suggests that router flips are either rare on GSM8K at the scale we test or are themselves concentrated in low-margin steps that C's gate already covers. For all three Tier~2 models, a body-scope method would need to be evaluated to close the residual gap.

\textbf{(4)~The same intervention ranking holds on every model where A, B, C were all run.} Integer quantization (A) has exactly zero effect, and Interventions B and C give comparable lifts. The ranking A $<$ B $\approx$ C is a robust empirical finding, independent of model family or scale.

\begin{figure}[h!]
\centering
\includegraphics[width=0.95\linewidth]{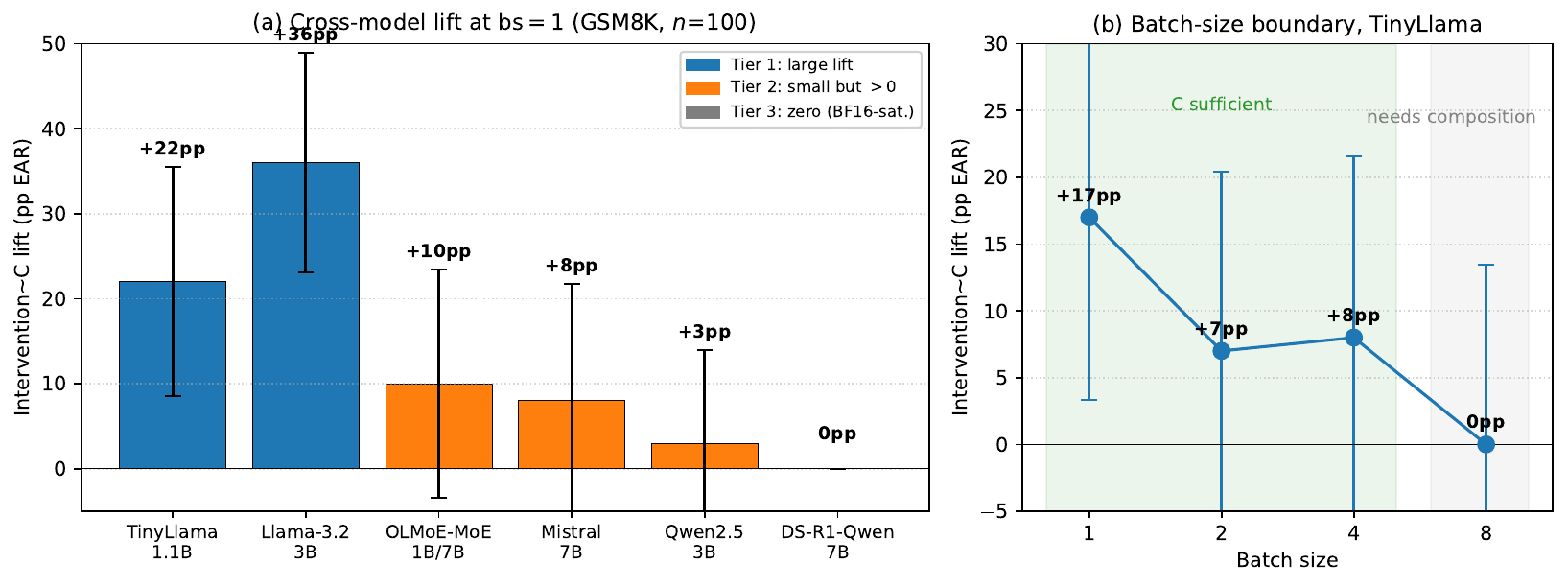}
\caption{Applicability map for Intervention~C. \textbf{(a)}~Cross-model lift at bs$=1$ for six models (including the MoE architecture OLMoE), coloured by tier (Table~\ref{tab:cross_model_sweep}): Tier 1 (Llama family, large lift), Tier 2 (non-Qwen $7$B, Qwen-3B, and OLMoE-MoE; small but positive), Tier 3 (Qwen BF16-saturated). Error bars are $95\%$ CIs on the paired lift. \textbf{(b)}~Batch-size scaling on TinyLlama-1.1B (Appendix~\ref{app:batch_size}): the lift holds at bs$\leq\!4$ (green band, where C alone is sufficient in this experiment) and vanishes at bs$\geq\!8$ (grey band, where composition with global FP32 compute recovers a small positive lift). Together, the panels summarise the tested $(model, bs)$ region in which the lm\_head-only intervention helps.}
\label{fig:scope_map}
\end{figure}

\paragraph{Implications for deployment.} Practitioners choosing between Intervention~C and body-scope alternatives (e.g., LayerCast) should first measure the target model's divergence rate and margin distribution on the target platform. Our cross-model results suggest three empirical regimes: (i)~\emph{Tier 1} (the tested Llama-family models at $1$--$3$B), where C delivers $+20$ to $+36$pp EAR; (ii)~\emph{Tier 2} (Qwen2.5-3B, Mistral-7B, and OLMoE), where C delivers smaller $+3$ to $+10$pp gains; and (iii)~\emph{Tier 3} (the tested BF16-saturated Qwen variants), where C is ineffective by itself. These tiers describe the observed models rather than a guaranteed family-level taxonomy. Appendix~\ref{app:batch_size} further reports $+17$pp at bs$=1$, $+7$--$8$pp at bs$\in\{2,4\}$, and $0$pp at bs$=8$; composition with global FP32 compute recovers $+5$pp at bs$=8$ in the tested setup.

\section{7B-scale scope boundary: DeepSeek-R1-Distill-Qwen-7B}
\label{app:scale_7b}

\paragraph{Setup.} To test whether the gated FP32 lm\_head recomputation (Intervention~C) transfers to a $7$B-parameter model, we rerun the full intervention table (baseline, A, B, C) on DeepSeek-R1-Distill-Qwen-7B, GSM8K, $n\!=\!100$, $256$-token budget, $\tauval\!=\!10^{-3}$, and the same hardware/software stack as the main experiments. Because the model's BF16 memory footprint is ${\sim}14$\,GB, we run BF16 and FP16 in sequential subprocesses to fit the 22\,GB GPU budget; each precision--strategy combination takes ${\sim}9.2$s per prompt, for a total ablation runtime of ${\sim}2$\,h.

\begin{table}[t]
\centering
\caption{7B intervention replication (DeepSeek-R1-Distill-Qwen-7B, GSM8K, $n\!=\!100$, $\tauval\!=\!10^{-3}$). All four arms collapse to $0\%$ EAR with identical Wilson CIs $[0.0, 3.7]\%$: every prompt diverges, and the interventions cannot recover. Latency columns show per-prompt BF16 / FP16 time; the ${\sim}5\%$ C overhead confirms the intervention is triggering, not silently skipped.}
\label{tab:scale_7b}
\small
\begin{tabular}{lccccc}
\toprule
Strategy & Match & EAR (\%) & 95\% CI & BF16 time (s) & FP16 time (s) \\
\midrule
Baseline      & 0/100 & 0.0 & $[0.0, 3.7]$ & 9.20 & 9.22 \\
A (int quant) & 0/100 & 0.0 & $[0.0, 3.7]$ & 9.40 & 9.17 \\
B (top-$K$ FP32) & 0/100 & 0.0 & $[0.0, 3.7]$ & 9.55 & 9.30 \\
C (full FP32 lm\_head) & 0/100 & 0.0 & $[0.0, 3.7]$ & 9.67 & 9.21 \\
\bottomrule
\end{tabular}
\end{table}

\paragraph{Interpretation: this is a scope boundary, not a method failure.} The $0\%$ EAR across all four arms is fully consistent with the paper's \rev{analysis} and screening data, and we read it as concrete evidence for where Intervention~C's scope ends.

\textbf{(1)~The first-divergence step occurs earlier than the intervention's correctable regime.} Table~\ref{tab:tstar_7b} breaks down the first-divergence step $t^\ast$ on the baseline BF16 vs.\ FP16 outputs. $100/100$ prompts diverge, with a median $t^\ast\!=\!6$ and $86/100$ prompts diverging within the first 9 tokens. This matches the pattern of Qwen2.5-7B-Instruct on the same benchmark ($t^\ast$ median $=1$, Section~\ref{sec:model_sensitivity}) and sits in the regime Appendix~\ref{app:theory} explicitly identifies as uncorrectable by lm\_head-only FP32 recomputation: when divergence happens during the early generation window, the $L$-layer error accumulated inside the transformer body dominates the single-step lm\_head error that Intervention~C targets, and the KV cache drifts irrecoverably within a few autoregressive steps.

\begin{table}[t]
\centering
\caption{First-divergence step distribution on DeepSeek-R1-Distill-Qwen-7B (GSM8K, $n\!=\!100$ baseline BF16 vs.\ FP16). Median $t^\ast\!=\!6$; $86\%$ of prompts diverge in the first 9 tokens, well within the ``step-0'' regime of Appendix~\ref{app:theory}.}
\label{tab:tstar_7b}
\small
\begin{tabular}{lc}
\toprule
First divergence bucket & Prompts \\
\midrule
$t^\ast = 0$          & 10 \\
$t^\ast \in [1, 4]$   & 21 \\
$t^\ast \in [5, 9]$   & 55 \\
$t^\ast \in [10, 29]$ & 14 \\
$t^\ast \geq 30$      & 0 \\
\midrule
AGREED ($t^\ast$ does not exist) & 0 \\
\bottomrule
\end{tabular}
\end{table}

\textbf{(2)~The intervention \emph{is} firing.} Latency for Intervention~C on BF16 rises from the baseline's $9.20$s to $9.67$s ($+5\%$), confirming that FP32 lm\_head recomputation is being triggered (the per-step top-two margin does drop below $\tauval\!=\!10^{-3}$ on this model, consistent with the small-margin phenomenology in Section~\ref{sec:div_cond_trace}). What fails is not the trigger, but the \emph{correctability} of the resulting flip: by the time a low-margin step arrives, the autoregressive path has already been driven onto a precision-specific KV cache. Appendix~\ref{app:kv_graft}'s finding that KV-cache swaps fix only $36$--$48\%$ of flips \emph{even on TinyLlama} is amplified here, because on this 7B model all 100 prompts sit on the cache-locked branch.

\textbf{(3)~This fits the paper's first-token screening narrative.} The model-sensitivity analysis already flags Qwen2.5-7B-Instruct and the Qwen-1.5B family as ``catastrophic first-token divergence'' models, and our 7B replication extends that list to DeepSeek-R1-Distill-Qwen-7B. These models share an observed inference-time pattern in our experiments: FP16 casting produces unstable, occasionally NaN intermediate activations, and the corresponding argmax trajectories bifurcate within a handful of decoding steps. We hypothesise that training-time precision exposure contributes to this pattern, but we do not verify the models' training histories or establish that causal mechanism. For this class of models, an lm\_head-scoped intervention is not sufficient; the present experiment establishes only that Interventions~A--C cannot recover agreement. Our TinyLlama-1.1B and Qwen2.5-3B-Instruct results (Section~\ref{sec:div_cond_trace}, Appendix~\ref{app:qwen3b_trace}), which do admit a non-empty \textsc{Agreed} group and do respond to Intervention~C, characterise where the method applies; the 7B result here characterises where it does not.

\paragraph{Practical takeaway.} Practitioners deploying a 7B-scale BF16-saturated model should not expect Intervention~C (or any other inference-time lm\_head-scoped fix) to reconcile BF16-vs-FP16 outputs. The first reproducibility measure for such models is to pin the serving precision. If cross-precision reproducibility remains required, a body-scope approach such as LayerCast or precision-aware retraining would need to be evaluated; we do not test those alternatives on this model. Our main-text method is designed for---and evaluated on---the substantial class of models where the divergence is a late, small-margin event at the lm\_head, which is exactly what makes it a single-step-correctable event.

\section{Quality preservation: does Intervention~C change task accuracy?}
\label{app:quality}

\paragraph{Setup.} Intervention~C reconciles BF16-vs-FP16 outputs at the lm\_head, but a natural concern is whether this comes at the cost of \emph{task accuracy}: perhaps the intervention pushes both precisions toward the same \emph{wrong} answer. We test this directly on GSM8K ($n\!=\!100$) by running four configurations---BF16 baseline, FP16 baseline, BF16 + C, FP16 + C---and extracting the final numerical answer from each generation. We use three testbeds spanning the three family/scale tiers of Table~\ref{tab:cross_model_sweep}: Qwen2.5-3B-Instruct (Tier 2, Qwen-family), Llama-3.2-3B-Instruct (Tier 1, largest EAR lift in the paper), and Mistral-7B-Instruct-v0.3 (Tier 2, non-Qwen 7B). All three models have GSM8K zero-shot accuracy well above the $\sim\!3\%$ noise floor we observe on TinyLlama,\footnote{All reported accuracies in this section use our bare-prompt format (\texttt{"Solve step by step:\textbackslash n\{question\}\textbackslash nAnswer:"}) with greedy decoding and no system prompt or chat template. This choice is deliberate: it is the same format used for the cross-family EAR comparison (Table~\ref{tab:strategy_cmp_cross}) and keeps the setup invariant across TinyLlama, Qwen, Llama, and Mistral families, so the EAR numbers are directly comparable. Absolute Pass@1 under this format is a lower bound on what each model can achieve with its native chat template. Because the quality-preservation tests are matched-pair McNemar tests over the same prompts in all four arms, the conclusions about Intervention~C's accuracy effect are unaffected by prompt-format choice; a chat-template replication confirms this empirically by repeating the Llama-3.2-3B rows with the native chat template, where absolute Pass@1 rises to $66\%$ but the McNemar tests remain non-significant ($p\!=\!0.625$ and $p\!=\!1.000$).} giving matched-pair tests enough statistical power to detect small accuracy shifts.

\begin{table}[t]
\centering
\caption{Per-precision task accuracy on GSM8K ($n\!=\!100$, final-answer extraction). McNemar exact paired tests compare Baseline vs.\ Intervention~C within each (model, precision) cell. No test rejects the null at any conventional level. Despite the Llama-3.2-3B family showing the largest EAR lift in the paper ($+36$pp, Table~\ref{tab:strategy_cmp_cross}), its task accuracy is unchanged under Intervention~C; despite Mistral-7B being the only $7$B non-Qwen model we tested, its Pass@1 point estimates drop by $3$--$5$pp but with non-significant McNemar $p$ values ($0.125$ and $0.250$) and substantially overlapping Wilson CIs. The six within-model tests (two per model) span EAR lifts from $+3$pp to $+36$pp and provide no evidence that Intervention~C systematically shifts either trajectory toward a wrong answer.}
\label{tab:qp}
\small
\begin{tabular}{llcccc}
\toprule
Model & Config & Correct & Acc (\%) & 95\% CI & vs.\ Baseline \\
\midrule
Qwen2.5-3B   & BF16, Baseline       & 27/100 & 27.0 & $[19.3, 36.4]$ & --- \\
             & BF16, Intervention~C & 29/100 & 29.0 & $[21.0, 38.5]$ & $\Delta$Acc$=+2$pp ($p\!=\!0.688$) \\
             & FP16, Baseline       & 28/100 & 28.0 & $[20.1, 37.5]$ & --- \\
             & FP16, Intervention~C & 27/100 & 27.0 & $[19.3, 36.4]$ & $\Delta$Acc$=-1$pp ($p\!=\!1.000$) \\
\midrule
Llama-3.2-3B & BF16, Baseline       & 24/100 & 24.0 & $[16.7, 33.2]$ & --- \\
             & BF16, Intervention~C & 23/100 & 23.0 & $[15.8, 32.1]$ & $\Delta$Acc$=-1$pp ($p\!=\!1.000$) \\
             & FP16, Baseline       & 20/100 & 20.0 & $[13.3, 28.9]$ & --- \\
             & FP16, Intervention~C & 21/100 & 21.0 & $[14.2, 30.0]$ & $\Delta$Acc$=+1$pp ($p\!=\!1.000$) \\
\midrule
Mistral-7B   & BF16, Baseline       & 43/100 & 43.0 & $[33.7, 52.8]$ & --- \\
             & BF16, Intervention~C & 38/100 & 38.0 & $[29.1, 47.8]$ & $\Delta$Acc$=-5$pp ($p\!=\!0.125$, $+1/-6$) \\
             & FP16, Baseline       & 40/100 & 40.0 & $[30.9, 49.8]$ & --- \\
             & FP16, Intervention~C & 37/100 & 37.0 & $[28.2, 46.8]$ & $\Delta$Acc$=-3$pp ($p\!=\!0.250$, $+0/-3$) \\
\bottomrule
\end{tabular}
\end{table}

\paragraph{Findings.}

\textbf{(1)~No statistically significant accuracy change in either direction, on any model.} Under matched-pair McNemar tests, no (model, precision) cell rejects the null at any conventional level. For Qwen2.5-3B, BF16 goes $+4 / -2$ ($p\!=\!0.688$) and FP16 $+1 / -2$ ($p\!=\!1.000$). For Llama-3.2-3B, BF16 goes $+8 / -9$ ($p\!=\!1.000$) and FP16 $+2 / -1$ ($p\!=\!1.000$). For Mistral-7B, BF16 goes $+1 / -6$ ($p\!=\!0.125$) and FP16 $+0 / -3$ ($p\!=\!0.250$); both $p$-values sit above the $0.05$ threshold but below $0.3$, and the point-estimate drops ($-5$pp on BF16, $-3$pp on FP16) are worth a dedicated discussion (see finding (3) below). The paper's $+22$pp (TinyLlama), $+36$pp (Llama-3.2-3B), $+3$pp (Qwen2.5-3B), and $+8$pp (Mistral-7B) EAR gains are therefore not artefacts of Intervention~C systematically pushing both precisions toward the same wrong answer.

\textbf{(2)~BF16 and FP16 baselines have essentially the same task accuracy within each model.} On Qwen2.5-3B, BF16 baseline Pass@1 is $27\%$ and FP16 baseline is $28\%$; on Llama-3.2-3B, BF16 is $24\%$ and FP16 is $20\%$; on Mistral-7B, BF16 is $43\%$ and FP16 is $40\%$. The within-model gaps ($\leq 4$pp) are within paired-sample uncertainty. Aggregate task accuracy therefore does not identify a uniformly preferable low-precision arm on disagreeing prompts. The separate FP32-fidelity analysis (Appendix~\ref{app:fp32_fidelity}) shows that FP16 is closer to the FP32 trajectory overall; Intervention~C is intended to improve cross-format replay when the serving format cannot simply be changed.

\textbf{(3)~Effect-size patterns across tiers: large EAR lift with balanced churn (Llama-3.2-3B) and small EAR lift with an uncertain accuracy shift (Mistral-7B).} Llama-3.2-3B (Tier 1, $+36$pp EAR) exhibits large but balanced within-precision correctness churn under C ($+8/-9$ on BF16). Mistral-7B (Tier 2, $+8$pp EAR) exhibits asymmetric small churn ($+1/-6$ on BF16, $+0/-3$ on FP16) that does not reach significance but points in one direction. The data are compatible with either sampling noise or a small negative accuracy effect; at $n\!=\!100$ we cannot distinguish them. Consequently, a deployment should evaluate both agreement and task quality on its own workload rather than assuming that a modest agreement gain justifies an uncertain quality trade-off. A larger-$n$ Mistral-7B follow-up remains useful.

\textbf{(4)~Small-model baseline for comparison.} We also ran the same experiment on TinyLlama-1.1B (Pass@1 around $3$--$4\%$ for all four arms, $n\!=\!100$). Matched-pair McNemar tests are again non-significant ($p>0.3$). However TinyLlama's accuracy sits at the noise floor for $n\!=\!100$ and the tests have limited power; the Qwen2.5-3B, Llama-3.2-3B, and Mistral-7B rows of the quality-preservation results are the primary quality-preservation evidence.

\textbf{(5)~Implication for the main claim.} The quality results show that Intervention~C improves BF16-vs-FP16 output agreement ($+22$pp EAR on TinyLlama, Table~\ref{tab:strategy_cmp}; $+36$pp on Llama-3.2-3B, Table~\ref{tab:strategy_cmp_cross}; with scaled agreement replication in Appendix~\ref{app:scaled_n}) without a statistically significant accuracy change on the four tested models. This is evidence of quality preservation in the tested conditions, not proof of zero effect: the Mistral point estimates and the finite sample sizes motivate workload-specific validation.

\paragraph{Prompt-format robustness check (Llama-3.2-3B).} The Llama-3.2-3B Pass@1 numbers in the quality-preservation results ($20$--$24\%$) sit well below the model's advertised GSM8K accuracy because the bare-prompt format used throughout our cross-family experiments does not invoke the chat template against which Llama-3.2-3B-Instruct was tuned. To confirm that the quality-preservation conclusion does not depend on this choice, we repeat the four-arm experiment with Llama-3.2-3B's native chat template (using a short system prompt: ``You are a careful math tutor. Solve the problem step by step. End your response with a line that contains the final numeric answer only.''), keeping every other setting identical (same $100$ GSM8K test items, greedy decoding, $\tauval\!=\!10^{-3}$, $\text{max\_new\_tokens}\!=\!512$).

\begin{table}[t]
\centering
\caption{Prompt-format robustness: the same four-arm experiment on Llama-3.2-3B-Instruct using its native chat template with a short system prompt. Absolute Pass@1 more than doubles ($24\%\!\to\!66\%$ on BF16 baseline), confirming that the lower numbers in the quality-preservation results reflect prompt format rather than model capability. Both McNemar paired tests remain non-significant, and FP16 shows zero Pass@1 flips under Intervention~C ($+0/-0$): the intervention reconciles outputs exactly. The quality-preservation conclusion is therefore invariant to prompt format.}
\label{tab:qp_chat}
\small
\begin{tabular}{lcccc}
\toprule
Config & Correct & Acc (\%) & 95\% CI & vs.\ Baseline \\
\midrule
BF16, Baseline       & 66/100 & 66.0 & $[56.3, 74.5]$ & --- \\
BF16, Intervention~C & 64/100 & 64.0 & $[54.2, 72.7]$ & $\Delta$Acc$=-2$pp ($p\!=\!0.625$, $+1/-3$) \\
FP16, Baseline       & 66/100 & 66.0 & $[56.3, 74.5]$ & --- \\
FP16, Intervention~C & 66/100 & 66.0 & $[56.3, 74.5]$ & $\Delta$Acc$=\phantom{-}0$pp ($p\!=\!1.000$, $+0/-0$) \\
\bottomrule
\end{tabular}
\end{table}

The chat-template results show two patterns. First, the FP16 arm produces \emph{identical} per-prompt correctness under Baseline and Intervention~C ($+0/-0$): on the $66$ prompts where FP16 answered correctly under baseline it still does so under C, and symmetrically for the incorrect prompts. The intervention therefore preserves which FP16 prompts are answered correctly while reconciling their token trajectories with the BF16 counterpart. Second, the BF16 arm shows a $2$pp drop ($-3/+1$, McNemar $p\!=\!0.625$), well within noise and with substantially overlapping Wilson CIs. Both arms are individually and jointly consistent with no accuracy change. Combined with the main quality-preservation table, we therefore have quality-preservation evidence on (Qwen2.5-3B, Llama-3.2-3B, and Mistral-7B with bare prompt, plus Llama-3.2-3B with chat template), across four model / prompt-format conditions, none of which detect a significant accuracy effect from Intervention~C.

\section{Long-reasoning sanity check: does divergence grow with chain length?}
\label{app:long_reasoning}

\paragraph{Motivation.} The main-text GSM8K experiments cap generation at $256$ tokens, which matches the typical reasoning-chain length for GSM8K problems. Modern LLM workloads, however, routinely generate chains of $500$--$4000$ tokens (competition math, agentic planning, long-context tool use). A natural question is whether the $30$--$64\%$ cross-precision divergence we report is a short-chain artefact that would either amplify or vanish under longer generations. The a~priori prediction was amplification: longer chains contain more low-margin steps, each one an opportunity for BF16 and FP16 to fork, and a single fork cascades. This appendix tests that prediction directly on MATH-500, a harder benchmark (Hendrycks-style competition math) at $\text{max\_new\_tokens}\!=\!1024$.

\paragraph{Setup.} We draw the first $n\!=\!50$ MATH-500 test problems deterministically. Prompt format is bare (\texttt{"Problem:\textbackslash n\{problem\}\textbackslash n\textbackslash nSolution:\textbackslash n"}), consistent with every other cross-family experiment in the paper---no chat template, no system prompt. We run baseline and Intervention~C on TinyLlama-1.1B-Chat and Llama-3.2-3B-Instruct (the two Tier~1 models from Table~\ref{tab:cross_model_sweep}) in both BF16 and FP16, batch size~$1$, fixed seed, A10G 22\,GB.

\begin{table}[t]
\centering
\caption{Long-reasoning sanity check on MATH-500 ($n\!=\!50$, $\text{max\_new\_tokens}\!=\!1024$, bare-prompt format). Mean generated length is $430$--$610$ tokens per prompt, $3$--$4\times$ longer than the main-text GSM8K chains. Baseline EAR remains close to the GSM8K values, and Intervention~C has a positive point-estimate lift on both models. The per-cell Wilson intervals are descriptive rather than paired intervals for the lifts.}
\label{tab:long_reasoning}
\small
\begin{tabular}{llcccc}
\toprule
Model & Variant & EAR (\%) & 95\% CI & Avg.\ len & Trig/step \\
\midrule
\multirow{2}{*}{TinyLlama-1.1B-Chat} & Baseline        & $44.0$ & $[31.2, 57.7]$ & $606$ & $0.00$ \\
                                     & Intervention~C  & $64.0$ & $[50.1, 75.9]$ & $564$ & $0.82$ \\
\midrule
\multirow{2}{*}{Llama-3.2-3B-Instruct} & Baseline       & $30.0$ & $[19.1, 43.8]$ & $438$ & $0.00$ \\
                                       & Intervention~C & $56.0$ & $[42.3, 68.8]$ & $430$ & $1.89$ \\
\bottomrule
\end{tabular}
\end{table}

\paragraph{Findings.}

\textbf{(1)~Baseline divergence does not amplify with chain length.} Compared to the GSM8K numbers in Table~\ref{tab:strategy_cmp_cross} (TinyLlama baseline $41\%$, Llama-3.2-3B baseline $31\%$), the MATH-500 baselines are $44\%$ and $30\%$ respectively---both within $3$pp of their GSM8K counterparts despite the mean chain length tripling ($170$ tokens on GSM8K $\to$ $438$--$606$ tokens on MATH-500) and the benchmark difficulty increasing substantially. This rules out the ``long chain cascades more'' hypothesis at the lengths we test.

\textbf{(2)~The saturation is consistent with the margin-based mechanism.} Section~\ref{sec:div_cond_trace} shows that divergence is associated with the density of low-margin lm\_head steps rather than chain length alone. MATH-500 has more gated steps per prompt than GSM8K (trigger rate $1.89$ vs.\ ${\sim}1$ on Llama-3.2-3B GSM8K), yet prompt-level divergence remains similar. Once a token differs, the autoregressive trajectories usually remain different; before that first differing token, additional decoding steps create further opportunities for a low-margin fork. Over the tested length range, the observed rates suggest that this prompt-level probability has already approached a plateau.

\textbf{(3)~Intervention~C's point-estimate lift persists at longer chains.} The lifts are $+20$pp (TinyLlama) and $+26$pp (Llama-3.2-3B) on MATH-500, compared with $+22$pp (TinyLlama, Table~\ref{tab:strategy_cmp}) and $+36$pp (Llama-3.2-3B, Table~\ref{tab:strategy_cmp_cross}) on GSM8K. The MATH-500 point estimates are positive and of similar order, but the $n\!=\!50$ per-cell Wilson intervals are descriptive and do not constitute paired confidence intervals for the lifts. The Llama-3.2-3B point estimate is lower on MATH-500; determining whether this reflects task difficulty, trigger density, or sampling variation would require a larger paired evaluation.

\textbf{(4)~An incidental latency observation.} On MATH-500, Llama-3.2-3B's FP16 baseline wall-clock time ($10.93$s/prompt) is about $24\%$ \emph{faster} than its BF16 counterpart ($14.46$s/prompt). This is the first latency observation of this magnitude in our experiments and is not relevant to the intervention claim, but we record it here for transparency. The most plausible explanation is that SM 86 (A10G) has a somewhat faster FP16 FMA path than its BF16 FMA path, plausibly amplified by the longer generation lengths on MATH-500; we did not pursue this further.

\paragraph{Summary.} The phenomenon and a positive Intervention~C point-estimate lift both persist on chains that are $3$--$4\times$ longer. The Wilson intervals are wider at $n\!=\!50$ than in the $n\!=\!100$ GSM8K comparison, so these results support qualitative transfer rather than equality of effect sizes. Over the range explored, divergence does not increase with reasoning length alone.

\section{Production-stack robustness}
\label{app:sanity_ext}
\label{app:serving_stack}
\label{app:batch_size}
\label{app:layercast_composition}
\label{app:overhead}

\begin{table}[h]
\centering\small
\caption{Production-stack robustness checks (TinyLlama-1.1B, GSM8K).}
\begin{tabular}{lll}
\toprule
Axis & Result & Implication \\
\midrule
MoE (OLMoE) & 100\% BF16 self-consistency & Kernel non-det.\ ruled out \\
FlashAttention-2 & EAR $-$1pp; C lift $+$18pp (= eager) & \rev{Consistent across both tested kernels} \\
Batch size & C lift: +17/+7/+7/0pp (bs 1/2/4/8) & Bounded to bs$\leq$4 \\
Global FP32 compute + C & +5pp at bs=8; noise at bs=16 & Composition viable \\
Overhead & +1.4\% latency, +11\% memory & Negligible cost \\
\bottomrule
\end{tabular}
\end{table}

\section{Hyperparameter and intervention ablations}
\label{app:scope}
\label{app:scope_main}
\label{app:sweep}
\label{app:topk_ablation}
\label{app:additional_baselines}

\begin{table}[h]
\centering\small
\caption{Ablation summary.}
\begin{tabular}{lll}
\toprule
Ablation & Result & Confirms \\
\midrule
Scope: RMSNorm+lm\_head & EAR drops 63\%$\to$44\% & lm\_head-only optimal \\
Threshold $\tau\in[10^{-4},10^{-2}]$ & EAR insensitive & Bimodal margin \\
Top-$K$: $K\in\{2,...,|V|\}$ & All identical EAR & Rank-1-vs-2 only \\
Temperature sharpening & Zero effect & Scale-invariant ratio \\
Consensus decoding & 100\% at $2\times$ cost & Trivial upper bound \\
\bottomrule
\end{tabular}
\end{table}

\section{Statistical validation and reproducibility}
\label{app:cis}
\label{app:bootstrap}
\label{app:scaled_n}

Wilson 95\% CIs have half-width $\pm$9--10pp at $n{=}100$ and describe the individual agreement rates. Paired bootstrap intervals describe the lifts; for example, the TinyLlama $+22$pp lift has a 90\% CI of $[+11, +33]$pp. Scaled replication at $n{=}300$ (GSM8K), $n{=}164$ (full HumanEval), and $n{=}257$ (full MBPP) confirms all $n{=}100$ point estimates fall inside the scaled CIs\rev{; the scaled lifts are $+16.4$pp (GSM8K), $+28.1$pp (HumanEval), and $+27.2$pp (MBPP), each with McNemar $p<0.001$, with per-cell CI half-widths reduced to $\pm 5.5$--$7.4$pp}.

\section{Reproducibility statement}
\label{app:reproducibility}

\paragraph{Environment and methodology.} The primary experiments use an NVIDIA A10G (22\,GB, Ampere SM~86) under PyTorch 1.13 / 2.1 with CUDA 11.7. Cross-hardware replications use an NVIDIA L4 (24\,GB, Ada SM~89) and A100-SXM4 (40\,GB, Ampere SM~80); the FP16/FP32 probe uses a T4 (16\,GB, Turing SM~75). We enable deterministic flags (\texttt{torch.use\_deterministic\_algorithms(True)}, fixed CUBLAS workspace) and use greedy decoding ($T{=}0$) with fixed prompt selection. Models are loaded from Hugging Face Hub without checkpoint modification, except for the explicitly described runtime precision transformations. EAR compares complete token-id sequences; task-level agreement uses final-answer extraction for GSM8K and execution outcomes for HumanEval. Confidence intervals are Wilson score intervals unless otherwise stated, and bootstrap CIs use 10,000 resamples. The prompt formats, sample sizes, hardware exceptions, and intervention hyperparameters ($\tauval{=}10^{-3}$ for BF16/FP16 and $K{=}8$ for Intervention~B) are reported in the corresponding sections; the FP8 threshold is separately calibrated as described in Appendix~\ref{app:fp8}.

\rev{\section{Directional body-error analysis}\label{sec:directional_body_error}

To test whether body-layer error contributes directionally to flips, we project the BF16-vs-FP16 logit-error vector onto the top-two decision direction $d = e_{v^{(1)}} - e_{v^{(2)}}$. For the empirical top-two decision rule evaluated here, we compare the magnitude of this projection with the smaller top-two margin across the two arms; exceeding that margin predicts a flip. This rule approximates the exact all-competitor condition in Proposition~\ref{prop:divergence_main}. Note on populations: this analysis measures margins at the \emph{first-divergence step} $t^\ast$ of each prompt (median $0.0$ for \textsc{Diverged}), whereas Table~\ref{tab:layer_div_cond} reports margins at each group's fixed measurement step under the divergence-conditioned trace protocol ($150\times$ gap); the two tables therefore quote different statistics of different step populations, and both show the same qualitative separation of two-plus orders of magnitude.

\begin{table}[h]
\centering
\caption{Directional body-error analysis (GSM8K, $n{=}100$ per model). The mean body projection onto the top-two direction is indistinguishable between Diverged and Agreed steps on both models, confirming that body error does not preferentially point toward the decision boundary for flipping steps. The margin, not the directional body projection, separates the two groups.}
\label{tab:directional}
\small
\setlength{\tabcolsep}{4pt}
\begin{tabular}{llccl}
\toprule
Model & Quantity & \textsc{Diverged} & \textsc{Agreed} & Interpretation \\
\midrule
\multirow{5}{*}{TinyLlama-1.1B}
 & Hidden-state L2 $\|e\|$ & 0.999 & 1.052 & Indistinguishable \\
 & Body projection (mean $\pm$ SD) & \rev{$0.025 \pm 0.023$} & \rev{$0.026 \pm 0.020$} & \rev{Equiv.\ ($p{=}0.014$)} \\
 & Signed dir.\ perturb.\ $\Delta z^{(1)}\!-\!\Delta z^{(2)}$ & 0.039 & 0.036 & Indistinguishable \\
 & Top-two margin (mean) & 0.015 & 3.93 & $260\times$ gap \\
 & Top-two margin (median) & 0.0 & 2.69 & --- \\
\addlinespace
\multirow{2}{*}{Qwen2.5-3B}
 & Body projection (mean $\pm$ SD) & \rev{$0.051 \pm 0.070$} & \rev{$0.049 \pm 0.043$} & \rev{Equiv.\ ($p{=}0.006$)} \\
 & Top-two margin (mean) & 0.035 & 6.05 & $170\times$ gap \\
\bottomrule
\end{tabular}
\end{table}

\rev{\paragraph{The body-error null is formally established, not merely unrejected.} Because a null result carries weight only if it is powered, we test equivalence rather than relying on a failure to reject. On TinyLlama-1.1B the difference in body projection between \textsc{Diverged} and \textsc{Agreed} is $-0.0013$ ($95\%$ CI $[-0.0099, +0.0073]$; Welch $t=-0.30$, $p=0.77$), and the two one-sided tests (TOST) procedure, with an equivalence margin of half a pooled SD (\rev{$0.011$ for TinyLlama, $0.032$ for Qwen}), rejects non-equivalence at $p=0.014$. On Qwen2.5-3B the difference is $+0.0015$ ($95\%$ CI $[-0.0216, +0.0247]$; $t=+0.13$, $p=0.90$), with TOST $p=0.006$. \rev{We use a half-SD margin as a conventional medium effect size; the point estimates differ by $<0.06$ SD, far inside it.} So the two groups are statistically \emph{equivalent} on the decision-direction projection, while their top-two margins differ by two to three orders of magnitude ($0.015$ vs.\ $3.93$; $0.035$ vs.\ $6.05$). The asymmetry between these two comparisons is the substantive content of this appendix.}

\rev{A note on what these accuracies measure.} Proposition~\ref{prop:divergence_main} is exact, so it cannot have an error rate; the quantity it involves---the per-coordinate perturbation at the step in question---is not observable before the step is computed. \rev{What we evaluate here is therefore a \emph{decision rule} derived from it: predict a flip when $|\Delta z(v^{(1)}) - \Delta z(v^{(2)})|$ exceeds the smaller of the two arms' margins, thresholding the \emph{magnitude} $|\Delta z(v^{(1)})-\Delta z(v^{(2)})|$ against the margin. Under this rule the directional statistic classifies flip-vs-no-flip correctly on $88\%$ of steps for TinyLlama-1.1B and $89\%$ for Qwen2.5-3B, versus step-level base rates of $63\%$/$55\%$ and versus $86\%$/$83\%$ for the RMS statistic it replaces. \rev{(The magnitude, not the signed difference, is the right predictor here: the signed rule scores only $0.57$/$0.59$, because a flip can be triggered by a perturbation of either sign relative to the current top-1.)} The exact condition of Proposition~\ref{prop:divergence_main} is proved; this decision rule is the deployable approximation to it, and the residual error is why we label the scale analysis heuristic. \rev{The intervention's own gate does not use this statistic at all---it fires on the observable margin $\Delta_t < \tauval$; the decision rule here is only an analysis tool for attributing flips.}} This confirms that (i)~the lm\_head margin is the dominant factor, (ii)~body error does not preferentially align with the decision boundary, and (iii)~the directional formulation of Proposition~\ref{prop:divergence_main} is empirically supported.
}

\rev{\section{FP32 fidelity of Intervention~C}\label{app:fp32_fidelity}

To verify that Intervention~C pushes outputs toward the FP32-oracle trajectory (rather than toward an arbitrary third path), we measure agreement between Intervention~C outputs and FP32 greedy outputs, comparing against the BF16-baseline-to-FP32 agreement rate.

\begin{table}[h]
\centering
\caption{FP32 reference fidelity (GSM8K, $n{=}100$). Three metrics: (1)~sequence-level agreement with the FP32 oracle, (2)~repaired-token agreement with the FP32 argmax at triggered steps, and (3)~downstream-correctness alignment: the fraction of FP32-correct prompts that are also answered correctly under each method. Intervention~C raises sequence agreement with FP32 by $+19$--$21$pp. \rev{Brackets on the sequence-agreement rows are $95\%$ CIs for the \emph{lift}; on the repaired-token rows they are Wilson $95\%$ CIs for the proportion.}}
\label{tab:fp32_fidelity}
\small
\begin{tabular}{llccc}
\toprule
Model & Metric & BF16$\to$FP32 & FP16$\to$FP32 & IntC$\to$FP32 \\
\midrule
\multirow{3}{*}{TinyLlama-1.1B}
  & Sequence agreement & $42\%$ & $90\%$ & $\mathbf{63\%}$ ($+21$pp) \rev{\scriptsize[+7,\,+35]} \\
  & Repaired-token = FP32 argmax & --- & --- & $63\%$ (97/154) \rev{\scriptsize[55,\,70]} \\
  & Correctness alignment w/ FP32 & $50\%$ & $100\%$ & $50\%$ \\
\addlinespace
\multirow{3}{*}{Qwen2.5-3B}
  & Sequence agreement & $23\%$ & $82\%$ & $\mathbf{42\%}$ ($+19$pp) \rev{\scriptsize[+6,\,+32]} \\
  & Repaired-token = FP32 argmax & --- & --- & $51\%$ (132/260) \rev{\scriptsize[45,\,57]} \\
  & Correctness alignment w/ FP32 & $25\%$ & $92\%$ & $\mathbf{50\%}$ \\
\bottomrule
\end{tabular}
\end{table}

On both models, Intervention~C substantially increases FP32 fidelity: sequence agreement rises from $42\%$ to $63\%$ on TinyLlama ($+21$pp) and from $23\%$ to $42\%$ on Qwen2.5-3B ($+19$pp). At triggered steps, $63\%$ (TinyLlama) and $51\%$ (Qwen) of repaired tokens match the FP32 argmax. Downstream-correctness alignment with FP32 doubles on Qwen2.5-3B ($25\% \to 50\%$); on TinyLlama the alignment is unchanged at $50\%$, though we caution that only $2$ of $100$ prompts are FP32-correct on this model (GSM8K accuracy ${\sim}2\%$), so this cell carries little statistical weight. FP16 already shows high FP32 fidelity ($90\%$/$82\%$) because FP16's 10-bit mantissa introduces smaller rounding errors; Intervention~C moves BF16 outputs measurably toward this higher-precision reference rather than toward an arbitrary consensus. \rev{Two qualifications are important. First, FP16 alone attains higher FP32 fidelity than BF16$+$C: if maximal FP32 fidelity is the objective and the serving format is free, the appropriate action is to serve FP16. Our claim is narrower---within a fixed BF16 deployment, the intervention moves outputs toward the FP32 reference. Second, because a repaired step is a choice between exactly two candidates, the chance level for repaired-token agreement is $50\%$: TinyLlama's $63\%$ (CI $[55,70]$) is above chance, whereas Qwen's $51\%$ (CI $[45,57]$) is not distinguishable from it. Per-token fidelity is therefore established on TinyLlama only; on Qwen the sequence-level ($+19$pp) and correctness-alignment ($25\% \to 50\%$, on 12 FP32-correct prompts) gains are the operative evidence.}
}

\rev{\section{Cross-architecture replication: L4 (Ada), A100 (Ampere), and T4 (Turing)}\label{app:l4_replication}

To verify that the mechanism and intervention are not A10G-specific, we replicate the headline experiment on an NVIDIA L4 (Ada Lovelace, SM~89, 24\,GB) \rev{and an NVIDIA A100-SXM4-40GB (Ampere, SM~80)}, and additionally probe the Turing generation via an NVIDIA T4 (SM~75, 16\,GB).

\begin{table}[h]
\centering
\caption{\rev{Cross-architecture replication (GSM8K, $n{=}100$, $\tauval{=}10^{-3}$). The intervention yields a positive lift on all three GPU microarchitectures tested; the \emph{magnitude} is hardware-dependent, and A10G is the most favourable card for TinyLlama and the least favourable for Qwen2.5-3B.}}
\label{tab:l4_replication}
\small
\begin{tabular}{llccc}
\toprule
Model & GPU (architecture) & Baseline EAR (\%) & IntC EAR (\%) & Lift \\
\midrule
\multirow{3}{*}{TinyLlama-1.1B}
 & A10G (Ampere, SM~86) & 41 & 63 & $+22$pp \\
 & L4 (Ada, SM~89) & 36 & 57 & $+21$pp \\
 & \rev{A100 (Ampere, SM~80)} & \rev{39} & \rev{51} & \rev{$+12$pp} \\
\addlinespace
\multirow{3}{*}{Qwen2.5-3B-Instruct}
 & A10G (Ampere, SM~86) & 18 & 21 & $+3$pp \\
 & L4 (Ada, SM~89) & 30 & 50 & $+20$pp \\
 & \rev{A100 (Ampere, SM~80)} & \rev{20} & \rev{34} & \rev{$+14$pp} \\
\bottomrule
\end{tabular}
\end{table}

\rev{The intervention transfers to every microarchitecture we tested, but the magnitude of its benefit does not, and we state that plainly. Two readings follow from the three-GPU comparison.

First, \emph{the direction of the effect is robust; the size is not.} All six model$\times$GPU cells show a positive lift, but TinyLlama ranges from $+22$pp (A10G) to $+12$pp (A100) and Qwen2.5-3B from $+3$pp (A10G) to $+20$pp (L4). The headline $+22$--$36$pp figures reported elsewhere in this paper are therefore A10G measurements and should be read as one point in a hardware-dependent range, not as architecture-independent constants.

Second, \emph{the A10G result is the outlier in this three-GPU comparison.} Qwen's lift is $+3$pp on A10G, $+14$pp on A100, and $+20$pp on L4. Because A100 (SM~80) and A10G (SM~86) are both Ampere yet differ by $11$pp, and because the two non-A10G cards agree more closely with each other, the variation cannot be attributed simply to the Ada architecture or the Qwen checkpoint. One possible mechanism is that A10G's FP16 kernel path pushes more activations outside the low-margin, lm\_head-reachable regime; we have not instrumented the kernels to test this explanation.

The practical consequence is that the tier assignments and the CV(TD) screening metric elsewhere in this paper are calibrated on A10G and require per-platform recalibration; Qwen2.5-3B is Tier~2 on A10G but would be Tier~1 on L4.}

\paragraph{A third generation: T4 (Turing).} The T4 cannot run this paper's primary comparison at all: Turing (SM~75) predates native BF16 support, which begins with Ampere (SM~80). Running BF16 there would measure software emulation rather than hardware format behaviour, and FP8 requires SM~89. The BF16/FP16 comparison therefore applies to Ampere-and-later serving stacks; on T4, FP16-vs-FP32 is the meaningful cross-precision pair, which we ran (TinyLlama-1.1B, GSM8K, $n{=}100$, $\tauval{=}10^{-3}$):

\begin{table}[h]
\centering
\caption{T4 (Turing, SM~75) FP16-vs-FP32 replication. Because FP32 holds the \emph{original} weights, this pair carries no weight-truncation floor, and baseline agreement is correspondingly far higher than in the BF16-vs-FP16 setting ($88\%$ vs.\ $41\%$). Gated FP32 lm\_head recomputation still yields a positive lift while triggering on only $20$ steps.}
\label{tab:t4_replication}
\small
\begin{tabular}{lccc}
\toprule
Precision pair & Baseline EAR (\%) & IntC EAR (\%) & Lift \\
\midrule
FP16 vs.\ FP32 (T4) & 88 & 95 & $+7$pp \\
\bottomrule
\end{tabular}
\end{table}

Two observations. First, the much higher baseline agreement is what the storage-vs-computation decomposition predicts: FP32 is the untruncated weight, so the ${\sim}12$pp irreducible floor identified for BF16-vs-FP16 (Section~\ref{sec:store_vs_compute}) does not apply here, and the residual disagreement is arithmetic-only. Second, the intervention remains effective in this regime ($+7$pp from $20$ triggered steps), so the low-margin mechanism is not an artefact of any single format pair or microarchitecture: across Turing, Ampere, and Ada, and across three precision pairs (FP16/FP32, BF16/FP16, BF16/FP8), gated FP32 lm\_head recomputation improves exact agreement. We report a single model here, as the T4's 16\,GB limits FP32-arm capacity.
}

\rev{\section{FP8 replication and threshold scaling}\label{app:fp8}

To test whether the mechanism and intervention extend beyond the BF16/FP16 pair to a lower-precision serving format, we replicate the headline experiment with FP8. We run on an NVIDIA L4 GPU (Ada Lovelace, compute capability 8.9, natively FP8-capable) and apply FP8-E4M3 with per-tensor dynamic scaling to the lm\_head projection, keeping the transformer body in BF16 in \emph{both} arms so that the comparison isolates the head-projection format---consistent with the finding that the lm\_head is the decisive locus (Section~\ref{sec:analysis}). The setup otherwise matches the main experiments: GSM8K, $n{=}100$, greedy decoding, $256$-token budget, on TinyLlama-1.1B-Chat and Qwen2.5-3B-Instruct.

\begin{table}[h]
\centering
\caption{Baseline FP8 divergence (GSM8K, $n{=}100$; lm\_head in FP8-E4M3 vs.\ BF16, transformer body in BF16 in both arms). FP8 divergence is substantially more severe than the FP16 setting ($70\%$ vs.\ $59\%$ of prompts on TinyLlama; $93\%$ vs.\ $75\%$ on Qwen), and flips still concentrate at low-margin steps. The FP16 reference column uses the same bare-prompt protocol as this experiment (see the cross-task appendix footnote for the protocol difference from the main-text chat-template numbers).}
\label{tab:fp8_baseline}
\small
\begin{tabular}{lccc}
\toprule
Model & EAR (BF16 vs FP8) & EAR (BF16 vs FP16) [reference] & First-div margin (median) \\
\midrule
TinyLlama-1.1B & $30\%$ & $41\%$ & $0.0625$ \\
Qwen2.5-3B & $7\%$ & $25\%$ & $0.125$ \\
\bottomrule
\end{tabular}
\end{table}

FP8's 3-bit mantissa produces larger rounding perturbations than FP16's 10 bits, so divergence rises on both models; yet the first-divergence margins remain small relative to the format's perturbation scale, showing that the low-margin lm\_head pattern also appears in the tested head-isolated BF16/FP8 setting.

\paragraph{Threshold scaling.} The FP16-calibrated trigger threshold $\tauval = 10^{-3}$ \emph{fails} under FP8: the gate opens on only $3$ of ${\sim}25{,}600$ decoding steps, and the lift ($-1$pp) is within noise. This failure is predicted by the scale analysis accompanying Proposition~\ref{prop:divergence_main}: FP8-E4M3's per-logit perturbation scale is roughly $2$--$5\times$ the FP16 scenario's $\sigma_z \approx 0.026$, so a threshold calibrated to the FP16 perturbation scale sits far below the margins at which FP8 flips occur. The trigger threshold must scale with the format's perturbation scale. Rescaling $\tauval$ recovers and then completes the repair (TinyLlama, both-arms gated protocol identical to the main experiments):

\begin{table}[h]
\centering
\caption{Threshold sweep for gated FP32 lm\_head recomputation under FP8 (GSM8K, $n{=}100$ per model, both-arms gated protocol). The FP16-calibrated $\tauval = 10^{-3}$ triggers almost never on either model; rescaling $\tauval$ to the FP8 perturbation scale monotonically recovers agreement on both, reaching exact agreement on TinyLlama ($100\%$) and $63\%$ on Qwen at $\tauval = 0.25$, each at a ${\sim}4\%$ trigger rate. \rev{Bracketed values are Wilson $95\%$ CIs; the $\tauval=0.25$ lifts are $+70$pp (CI $[+61, +79]$) and $+56$pp (CI $[+45, +67]$).}}
\label{tab:fp8_threshold}
\small
\begin{tabular}{lcccccc}
\toprule
 & \multicolumn{3}{c}{TinyLlama-1.1B} & \multicolumn{3}{c}{Qwen2.5-3B} \\
\cmidrule(lr){2-4} \cmidrule(lr){5-7}
$\tauval$ & EAR (\%) & Lift (pp) & Trig.\ steps & EAR (\%) & Lift (pp) & Trig.\ steps \\
\midrule
$10^{-3}$ & 29 \rev{\scriptsize[21,\,39]} & $-1$ & 3 & 10 \rev{\scriptsize[6,\,17]} & $+3$ & 4 \\
$0.02$ & 37 \rev{\scriptsize[28,\,47]} & $+7$ & 73 & 13 \rev{\scriptsize[8,\,21]} & $+6$ & 94 \\
$0.05$ & 48 \rev{\scriptsize[38,\,58]} & $+18$ & 201 & 16 \rev{\scriptsize[10,\,24]} & $+9$ & 197 \\
$0.1$ & 77 \rev{\scriptsize[68,\,84]} & $+47$ & 378 & 27 \rev{\scriptsize[19,\,36]} & $+20$ & 396 \\
$0.25$ & \textbf{100} \rev{\scriptsize[96,\,100]} & $+70$ & 1010 (${\sim}4\%$) & \textbf{63} \rev{\scriptsize[53,\,72]} & $+56$ & 994 (${\sim}4\%$) \\
\bottomrule
\end{tabular}
\end{table}

\rev{\paragraph{Why the threshold must scale: the flip-margin distribution moves with the format.}
The scale analysis predicts that flips concentrate where the margin is small \emph{relative to the format's perturbation scale}---so a more aggressive format should produce flips at correspondingly \emph{larger} margins. This is directly observable in the first-divergence margins:

\begin{table}[h]
\centering
\caption{Top-two margin at the first divergence step, by format pair (GSM8K, $n{=}100$ per cell). Under FP8 the flips occur at margins $2.4$--$4.8\times$ larger than under FP16, matching the larger per-logit perturbation of the 3-bit-mantissa format and explaining why a threshold calibrated for FP16 ($\tauval = 10^{-3}$) almost never fires under FP8.}
\label{tab:margin_by_format}
\small
\begin{tabular}{llcc}
\toprule
Format pair & Model & Flip-margin median & Flip-margin mean \\
\midrule
BF16 vs.\ FP16 & TinyLlama-1.1B & 0.000 & 0.015 \\
BF16 vs.\ FP16 & Qwen2.5-3B & 0.000 & 0.035 \\
\addlinespace
BF16 vs.\ FP8 & TinyLlama-1.1B & \textbf{0.0625} & 0.052 \\
BF16 vs.\ FP8 & Qwen2.5-3B & \textbf{0.1250} & 0.126 \\
\bottomrule
\end{tabular}
\end{table}

The shift is quantitative, not merely directional: relative to the FP16-setting scale $\sigma_z \approx 0.026$, the FP8 flip-margin medians sit at $2.4\times$ (TinyLlama) and $4.8\times$ (Qwen) that value, consistent with the $2$--$5\times$ larger per-logit perturbation of FP8-E4M3. The gate must widen by a comparable factor, which is what the sweep finds empirically ($\tauval$ from $10^{-3}$ to $0.25$). We note this is a two-format comparison; whether the relationship is proportional across a wider range of formats is untested.}

At $\tauval = 0.25$ the intervention achieves \emph{exact} agreement (EAR $100\%$) on TinyLlama at a ${\sim}4\%$ trigger rate---comparable overhead to the FP16 setting. The rescaling behaviour replicates on Qwen2.5-3B: the same sweep recovers monotonically from $+3$pp at $\tauval = 10^{-3}$ to $+56$pp at $\tauval = 0.25$ (EAR $7\% \to 63\%$, at a comparable ${\sim}4\%$ trigger rate), though it does not reach exact agreement. \rev{Note that the explanation cannot be body-originated divergence, which this design excludes: with identical BF16 bodies in both arms the shortfall must come from gate coverage. Qwen's FP8 flip-margin distribution has a heavier tail (median $0.125$ versus TinyLlama's $0.0625$, Table~\ref{tab:margin_by_format}), so $\tauval = 0.25$ leaves more of its flippable steps ungated; a larger $\tauval$ would be needed, at proportionally higher trigger cost.} \rev{\paragraph{End-to-end FP8: where the method's reach ends.} The configuration above deliberately isolates the head projection. We also ran the deployment-shaped variant, quantising \emph{every} \texttt{nn.Linear} in the model to FP8-E4M3 (155 modules; TinyLlama-1.1B, GSM8K, $n{=}100$). Divergence becomes far more severe --- EAR $10\%$, i.e.\ $90\%$ of prompts diverge --- and the gated repair recovers only $+1$pp. Widening the gate does not help: $\tauval = 0.1$ and $\tauval = 0.25$ trigger on $321$ and $738$ steps respectively and both yield the same $+1$pp. This is the boundary the error model predicts rather than a contradiction of it: with all linear layers quantized, divergence is dominated by body-originated error accumulated across the forward pass, which lies outside the reach of any lm\_head-scope repair. It is the same limit we report for BF16-saturated Qwen variants (Appendix~\ref{app:scale_7b}) and for bs$\geq$8 batched decode (Appendix~\ref{app:batch_size}): where lm\_head-originated divergence is not the dominant source, composition with a body-scope method is required. The contrast between the two FP8 configurations ($+70$pp when the body is held fixed, $+1$pp when it is not) is consistent with the head-dominance estimate of Section~\ref{sec:analysis}: when the body is quantized too, the body term the estimate treats as small ($\sim$2.4\% under BF16/FP16) grows until it dominates, and a head-scope repair can no longer reach it.}

\rev{The two FP8 configurations establish different claims. Because the head-isolated comparison shares BF16 weights and a BF16 transformer body, the hidden state entering the head is identical across arms; this removes both the measured BF16/FP16 weight-storage gap and body-originated divergence by construction. It therefore tests transfer of (i)~the low-margin lm\_head mechanism, (ii)~the margin-gated repair, and (iii)~format-dependent threshold calibration to FP8 head arithmetic. The end-to-end experiment quantises every linear layer and is the deployment-shaped stress test included here. Its $+1$pp result shows that the head-only method is largely ineffective once FP8 error accumulates through the body. We do not evaluate FP8 KV-cache quantisation or claim an end-to-end FP8 reproducibility solution.}
}

\rev{\section{Execution-based correctness (HumanEval)}\label{app:execution}

We run the OpenAI HumanEval execution sandbox on all four arms (BF16 baseline, FP16 baseline, BF16+C, FP16+C) using TinyLlama-1.1B ($n{=}164$, full benchmark).

\begin{table}[h]
\centering
\caption{Execution-based pass@1 on HumanEval (TinyLlama-1.1B, $n{=}164$). Despite $68\%$ token-level divergence between BF16 and FP16, execution correctness is identical across all configurations: zero correctness flips.}
\label{tab:execution}
\small
\begin{tabular}{lcc}
\toprule
Config & pass@1 (\%) & Correctness flips vs.\ BF16 baseline \\
\midrule
BF16 baseline & 0.6 & --- \\
FP16 baseline & 0.6 & 0 \\
BF16 + IntC & 0.6 & 0 \\
FP16 + IntC & 0.6 & 0 \\
\midrule
Token-level divergence (BF16 vs FP16) & \multicolumn{2}{c}{68\%} \\
\bottomrule
\end{tabular}
\end{table}

The $68\%$ token-level divergence produces zero execution-correctness flips: the divergent code completions are either both wrong or both right (in all observed cases, both wrong---TinyLlama's code accuracy is low). Intervention~C introduces zero pass@1 degradation in this run, but the near-zero baseline accuracy makes this a weak quality-preservation test; the Qwen experiment below is the informative execution-based result.

\paragraph{A capable model: Qwen2.5-3B-Instruct.} TinyLlama's near-zero code accuracy makes its ``zero flips'' result weak evidence: two arms that are both always wrong agree trivially. We therefore repeated the experiment on Qwen2.5-3B-Instruct, which solves half of HumanEval, so that agreement is measured over a population where correctness genuinely varies (Table~\ref{tab:execution_qwen}).

\begin{table}[h]
\centering
\caption{Execution-based pass@1 on HumanEval (Qwen2.5-3B-Instruct, $n{=}164$, full benchmark). At a pass@1 of $\sim$50\%, the two arms disagree on $73.8\%$ of completions at the token level, yet agree on $97.6\%$ of execution outcomes. Intervention~C does not degrade pass@1.}
\label{tab:execution_qwen}
\small
\begin{tabular}{lcc}
\toprule
Config & pass@1 (\%) & Correctness flips vs.\ BF16 baseline \\
\midrule
BF16 baseline & 50.0 & --- \\
FP16 baseline & 52.4 & 4 of 164 ($2.4\%$) \\
BF16 + IntC & 51.2 & --- ($+1.2$pp vs.\ BF16) \\
\midrule
Token-level divergence (BF16 vs FP16) & \multicolumn{2}{c}{$73.8\%$} \\
Execution-outcome agreement & \multicolumn{2}{c}{$97.6\%$} \\
\bottomrule
\end{tabular}
\end{table}

Three observations. First, the token-level/outcome-level gap is large and in the direction our discussion of practical adequacy predicts: $73.8\%$ of completions differ somewhere in their token sequence, but $97.6\%$ of them execute to the same verdict, so most cross-precision divergence on this task is not correctness-bearing. Second, unlike TinyLlama, the flips here are not zero: $4$ of $164$ problems flip, all in the same direction (FP16 correct, BF16 incorrect; exact binomial $p{=}0.125$, not significant at $n{=}164$). We report the direction rather than suppress it, but with $4$ discordant pairs the data cannot support a claim that either format is systematically better on code, and the aggregate $2.4$pp pass@1 difference is well inside the $\pm5$pp non-inferiority margin used elsewhere in this paper. Third, Intervention~C's pass@1 ($51.2\%$) is above the BF16 baseline it repairs toward ($50.0\%$), so the gated recomputation does not trade correctness for reproducibility on this benchmark either.

Across the two models, execution-outcome agreement greatly exceeds token-level agreement: $100\%$ vs.\ $32\%$ on TinyLlama and $97.6\%$ vs.\ $26.2\%$ on Qwen2.5-3B.
}

\rev{\section{Cross-task EAR on Qwen2.5-3B-Instruct}\label{app:cross_task_ear}

To verify that the phenomenon generalises beyond GSM8K, we report EAR on two additional code benchmarks for Qwen2.5-3B-Instruct.

\begin{table}[h]
\centering
\caption{Cross-task EAR for Qwen2.5-3B-Instruct ($n{=}100$, $\tauval{=}10^{-3}$). Divergence rates are consistent across math and code tasks.}
\label{tab:cross_task_ear}
\small
\begin{tabular}{lccc}
\toprule
Task & EAR (\%) & Divergence (\%) & Median first-div step \\
\midrule
GSM8K & 25 & 75 & --- \\
HumanEval & 37 & 63 & step 83 \\
MBPP & 40 & 60 & step 43 \\
\bottomrule
\end{tabular}
\end{table}

Both code benchmarks show $60$--$63\%$ divergence on Qwen2.5-3B-Instruct, comparable to the $75\%$ GSM8K divergence measured under the same protocol. The slightly higher EAR on code tasks is consistent with their shorter average generations and fewer observed low-margin steps per prompt.\footnote{\rev{The experiments in this appendix, together with the FP8, fidelity, and directional analyses, use a bare ``Solve step by step'' prompt format, under which Qwen2.5-3B GSM8K EAR is $25\%$ ($75\%$ divergence); the main-text cross-model table uses the chat-template protocol, giving $18\%$ EAR ($82\%$ divergence). Per-format numbers are therefore not directly comparable.}}
}

\end{document}